\documentclass[twoside,11pt]{article}

\usepackage{jmlr2e}
\usepackage{amsmath,amssymb,cite,graphicx,array,verbatim,multirow,float,mathrsfs}

\usepackage[shortlabels]{enumitem}
\usepackage{algorithm}
\usepackage{algorithmic}
\usepackage{natbib}
\usepackage{hyperref}

\ShortHeadings{Spectral MLE:  Top-$K$ Rank Aggregation from Pairwise Comparisons}{Chen and Suh}
\firstpageno{1}

\begin{document}

\title{Spectral MLE: \\  Top-$K$ Rank Aggregation from Pairwise Comparisons}

\author{\name Yuxin Chen \email yxchen@stanford.edu \\
       \addr Department of Statistics\\
       Stanford University\\
       Stanford, CA 94305, USA
       \AND
       \name Changho Suh \email chsuh@kaist.ac.kr \\
       \addr Department of Electrical Engineering\\
       Korea Advanced Institute of Science and Technology\\
       Daejeon, Korea}

\editor{}

\maketitle

\begin{abstract}
This paper explores the preference-based top-$K$ rank aggregation problem.
Suppose that a collection of items is repeatedly compared in pairs,
and one wishes to recover a consistent ordering that emphasizes the
top-$K$ ranked items, based on partially revealed preferences. For concreteness, this work focuses on the popular
Bradley-Terry-Luce model that postulates a set of latent
preference scores underlying all items, where the odds of paired comparisons
depend only on the relative scores of the items involved. 

We characterize the minimax limits on the identifiability of top-$K$ ranked items, in
the presence of random and non-adaptive sampling. Our findings highlight a separation
measure that quantifies the gap of preference scores between the $K^{\text{th}}$
and $(K+1)^{\text{th}}$ ranked items. In an attempt to approach
this minimax limit, we propose a nearly linear-time ranking scheme,
called \emph{Spectral MLE}, that returns the indices of the top-$K$
items in accordance to a careful score estimate. In a nutshell, Spectral
MLE starts with an initial score estimate with minimal squared loss
(obtained via a spectral method), and then successively refines each
component with the assistance of coordinate-wise MLEs. Encouragingly,
Spectral MLE achieves perfect top-$K$ item identification with minimal sample
complexity. The practical applicability of Spectral MLE is further corroborated by
numerical experiments.
\end{abstract}

\begin{keywords}
   Bradley-Terry-Luce model, top-$K$ ranking, linear-time algorithm, minimax limits, coordinate-wise MLE
\end{keywords}

\section{Introduction and Motivation}



The task of rank aggregation is encountered in a wide spectrum of
contexts like social choice \citep{caplin1991aggregation,AzariParkesXia2014},
web search and information retrieval \citep{brin1998anatomy,Dwork2001}, crowd sourcing
\citep{chen2013pairwise}, recommendation systems \citep{baltrunas2010group},
to name just a few. Given partial preference information over a collection
of items, the aim is to identify a consistent ordering that best respects
the revealed preference. In the high-dimensional regime, one is often
faced with two challenges: (1) the number of items to be ranked is
ever growing, which makes it increasingly harder to recover a consistent total ordering over
all items; (2) the observed data is highly incomplete and inconsistent:
only a small number of noisy pairwise / listwise preferences can be acquired.

In an effort to address such challenges, this paper explores a popular
pairwise preference-based model, which postulates the existence of
a ground-truth ranking. Specifically, consider a parametric model
involving $n$ items, each assigned a preference score that determines
its rank. Concrete examples of preference scores include the overall
rating of an athlete, the academic performance and competitiveness
of a university, the dining quality of a restaurant, etc. Each item
is then repeatedly compared against a few others in pairs, yielding
a set of \emph{noisy} binary comparisons generated based on the relative
preference scores. In many situations, the number of repeated comparisons
essentially reflects the signal-to-noise ratio (SNR) or the quality of
the information revealed for each pair of items. The goal is then
to develop a ``denoising'' procedure that recovers the ground-truth
ranking, ideally from a minimal number of noisy observations.

There has been a proliferation of ranking schemes 
that suggest partial solutions. While the ranking that we are seeking
is better treated as a function of the preference parameters, most
of the aforementioned schemes adopt the natural ``plug-in'' procedure,
that is, start by inferring the preference scores, and then return
a ranking in accordance to the parametric estimates. The most popular
paradigm is arguably the maximum likelihood estimation (MLE) \citep{ford1957solution},
where the main appeal of MLE is its inherent convexity under several
comparison models, e.g. the Bradley-Terry-Luce (BTL) model~\citep{bradley1952rank,luce1959individual}.
Encouragingly, MLE often achieves low $\ell_{2}$ estimation loss
while retaining efficient finite-sample complexity.
Another prominent alternative concerns a family of spectral ranking
algorithms (e.g. \emph{PageRank}~\citep{brin1998anatomy}). A provably
efficient choice within this family is \emph{Rank Centrality }\citep{Negahban2012},
which produces an estimate with nearly minimax mean squared error
(MSE). While both MLE and Rank Centrality allow intriguing guarantees
towards finding faithful parametric estimates, the squared loss metric
considered therein does not necessarily imply optimality of the ranking
accuracy. In fact, there is no shortage of high-dimensional situations
that admit parametric estimates with low squared loss while precluding
reliable ranking. Furthermore, many realistic scenarios
emphasize only a few items that receive the highest ranks. Unfortunately,
the above MSE results fall short of ensuring recovery of the top-ranked items.

In this work, we consider accurate identification of top-$K$ ranked
items under the popular BTL pairwise comparison model, assuming that
the item pairs we can compare are selected in a random and non-adaptive
fashion (termed \emph{passive ranking}). In particular, we aim to explore
 the following questions: (a) what is the minimum number of
repeated comparisons necessary for reliable ranking? (b) how is the
ranking accuracy affected by the underlying preference score distributions?
We will address these two questions from both statistical and algorithmic
perspectives.

\subsection{Main Contributions\label{sub:Contributions}}

This paper investigates minimax optimal procedures for top-$K$ rank
aggregation. Our contributions are two-fold.

To begin with, we characterize the fundamental three-way tradeoff between
the number of repeated comparisons, the sparsity of the comparison
graph, and the preference score distribution, from a minimax perspective.
In particular, we emphasize a separation measure that quantifies the
gap of preference scores between the $K^{\text{th}}$ and $(K+1)^{\text{th}}$ ranked items.
Our results demonstrate that the minimal sample complexity or 
quality of paired evaluation (reflected by the number of repeated
comparisons per an observed pair) scales inversely with the separation
measure at a quadratic rate.

Algorithmically, we propose a nearly linear-time two-stage procedure, called \emph{Spectral
MLE}, which allows perfect top-$K$ identification as soon as the sample
complexity exceeds the minimax limits (modulo some constant).
Specifically, Spectral MLE starts by obtaining careful initial scores
that are faithful in the $\ell_{2}$ sense (e.g. via a spectral ranking method),
and then iteratively sharpens the pointwise estimates via a careful comparison between
the running iterates and the coordinate-wise MLEs. This algorithm is designed primarily
in an attempt to seek a score estimate with minimal pointwise loss which, in turn, guarantees optimal ranking accuracy. 
  Furthermore, numerical experiments
demonstrate that Spectral MLE outperforms Rank Centrality by achieving lower $\ell_{\infty}$ estimation error and higher ranking accuracy. 


\subsection{Prior Art\label{sub:PriorArt}}

There are two distinct families of observation models that have received
considerable interest: (1) \emph{value-based model}, where the observation
on each item is drawn only from the distribution underlying this individual;
(2) \emph{preference-based model}, where one observes the relative order among a few items instead of revealing
their individual values. Best-$K$ identification in the value-based
model with adaptive sampling (termed {\em active ranking}) is closely related to the multi-armed
bandit problem, where the fundamental identification complexity has
been characterized \citep{gabillon2011multi,bubeck2013multiple,jamieson2013lil}. In addition, the value-based and preference-based models have been compared in terms of minimax error rates in estimating the latent quantities; see \citep{shah2014better}.

In the realm of pairwise preference settings, many \emph{active ranking}
schemes \citep{busa2014survey} have been proposed in an attempt to
optimize the exploration-exploitation tradeoff. For instance, in the
noise-free case, \citep{jamieson2011active} considered
perfect total ranking and characterized the query complexity gain
of adaptive sampling relative to random queries, provided that the
items under study admit a low-dimensional Euclidean embedding. Furthermore,
several works \citep{ailon2011active,jamieson2011active,braverman2008noisy, wauthier2013efficient}
explored the query complexity in the presence of noise, but were basically designed to recover
 ``approximately correct'' total rankings---a solution with
loss at most a factor $(1+\epsilon)$ from optimal---rather than accurate ordering. Another path-based approach has been proposed to
accommodate accurate top-$K$ queries from noisy pairwise data \citep{eriksson2013learning},
where the observation error is assumed to be i.i.d. instead of being item-dependent. Motivated by the success of value-based racing
algorithms, \citep{Hullermeier2013topKranking,busa2014survey}
came up with a generalized racing algorithm that often led to efficient
sample complexity. In contrast, the current paper concentrates on
top-$K$ identification in a \emph{passive} setting, assuming that
partial preferences are collected in a noisy, random, and non-adaptive
manner. This was previously out of reach.

Apart from Rank Centrality and MLE,
the most relevant work is \citep{rajkumar2014statistical}.
For a variety of rank aggregation methods, they developed intriguing
sufficient statistical hypotheses that guarantee the convergence to
an optimal ranking, which in turn leads to sample complexity bounds
for Rank Centrality and MLE. Nevertheless, they focused on perfect
total ordering instead of top-$K$ selection, and their results fall
short of a rigorous justification as to whether or not the derived
sample complexity bounds are statistically optimal.

Finally, there are many related yet different problem settings considered
in the prior literature.
For instance, the work \citep{ammar2012efficient} approached
top-$K$ ranking using a maximum entropy principle, assuming the existence
of a distribution $\mu$ over all possible permutations. Recent work  \citep{SoufianiChenParkesXioa2013,soufiani2014computing}
investigated consistent \emph{rank breaking} under more generalized
models involving full rankings.  
A family of distance measures on rankings has been studied and justified based on an axiomatic approach
\citep{farnoud2014axiomatic}.  Another line of works considered the popular distance-based  Mallows model \citep{lu2011learning,busa2014preference,awasthi2014learning}. An online ranking setting has been studied as well \citep{harrington2003online,farnoud2014approximate}. More broadly, the minimax recovery limits under general pairwise measurements have recently been  determined by \citep{chen2015information}.
These are beyond the scope of the present work.

\subsection{Organization and Notation\label{sub:Organization-and-Notation}}

The remainder of the paper is organized as follows. Section \ref{sec:Problem-Setup} introduces the pairwise comparison model as well as the key performance metrics for the top-$K$ ranking task. The main results,  including a fundamental minimax lower limit and an achievability result by nearly linear-time algorithms, are summarized and discussed in Section \ref{sec:Main-Results}.  Section \ref{sec:achievability} presents the detailed procedure and performance guarantees of the proposed Spectral MLE algorithm,  and provides a heuristic treatment as to why it is expected to control the $\ell_{\infty}$ estimation error. We conclude the paper with a summary of our findings and a discussion of about future research directions in  Section \ref{sec:Conclusion}. The proofs of the ranking performance of Spectral MLE (i.e.~Theorem \ref{thm:SpectralMLE}) and the minimax lower bound  (i.e.~Theorem \ref{thm:optimality}) are deferred to Appendix \ref{sec:Proof-SpectralMLE}  and Appendix \ref{sec:converse-proof}, respectively.

Before continuing, we provide a brief summary of the notations used throughout the paper.
Let $[n]$ represent $\left\{ 1,2,\cdots,n\right\} $. 
We denote by $\Vert \boldsymbol{w} \Vert $,
$\Vert \boldsymbol{w} \Vert _{1}$, $\Vert \boldsymbol{w} \Vert _{\infty}$
the $\ell_{2}$ norm, $\ell_{1}$ norm, and $\ell_{\infty}$ norm
of $\boldsymbol{w}$, respectively. 
A graph $\mathcal{G}$
is said to be an Erd\H{o}s\textendash{}R\'enyi random graph, denoted by $\mathcal{G}_{n,p_{{\rm obs}}}$,
if each pair $(i,j)$ is connected by an edge independently with
probability $p_{{\rm obs}}$. Besides, we use $\mathrm{deg}\left(i\right)$
to represent the degree of vertex $i$ in $\mathcal{G}$. 

Additionally,
for any two sequences $f(n)$ and $g(n)$, $f(n)\gtrsim g(n)$ or $f(n)=\Omega(g(n))$ mean
that there exists a constant $c$ such that $f(n)\geq cg(n)$; $f(n)\lesssim g(n)$ or $f(n) = O(g(n))$
mean that there exists a constant $c$ such that $f(n)\leq cg(n)$; and
$f(n)\asymp g(n)$ or $f(n)=\Theta(g(n))$ mean that there exist constants $c_{1}$ and $c_{2}$
such that $c_{1}g(n)\leq f(n)\leq c_{2}g(n)$.

\section{Problem Setup\label{sec:Problem-Setup}}

To formalize matters,  we present mathematical setups and key performance metrics in this section.

\noindent
\textbf{Comparison Model and Assumptions}. Suppose that we observe
a few pairwise evaluations over $n$ items. To pursue a statistical
understanding towards the ranking limits, we assume that the pairwise
comparison outcomes are generated according to the BTL model \citep{bradley1952rank,luce1959individual},
a long-standing model that has been applied in numerous
applications \citep{agresti2014categorical,hunter2004mm}. 
\begin{itemize}[leftmargin=*, listparindent =1em] 
\item \emph{Preference Scores}. The BTL model hypothesizes on the existence
of some hidden preference vector $\boldsymbol{w}=\left[w_{i}\right]_{1\leq i\leq n}$,
where $w_{i}$ represents the underlying preference score~/~weight
of item $i$. The outcome of each paired comparison depends only on
the scores of the items involved. Without loss of generality, we will
assume throughout that 
\begin{equation}
w_{1}\geq w_{2}\geq\cdots\geq w_{n}>0\label{eq:score-order}
\end{equation}
unless otherwise specified.
\item \emph{Comparison Graph}. Denote by $\mathcal{G}=\left(\left[n\right],\mathcal{E}\right)$
the comparison graph such that items $i$ and $j$ are
compared if and only if $(i,j)$ belongs to the edge set $\mathcal{E}$.
We will mostly assume that $\mathcal{G}$ is drawn
from the Erd\H{o}s\textendash{}R\'enyi model $\mathcal{G}\sim\mathcal{G}_{n,p_{\mathrm{obs}}}$
for some observation ratio $p_{\mathrm{obs}}$. 
\item \emph{(Repeated) Pairwise Comparisons}. For each $\left(i,j\right)\in\mathcal{E}$,
we observe $L$ independent paired comparisons between items $i$ and $j$. The
outcome of the $l^{\text{th}}$ comparison between them, denoted by
$y_{i,j}^{\left(l\right)}$, is generated as per the BTL model:
\begin{equation}
y_{i,j}^{\left(l\right)}=\begin{cases}
1,\quad & \text{with probability }\frac{w_{i}}{w_{i}+w_{j}},\\
0, & \text{else},
\end{cases}\label{eq:BTL}
\end{equation}
where $y_{i,j}^{\left(l\right)}=1$ indicates a win by $i$ over $j$.
We adopt the convention that $y_{j,i}^{(l)}=1-y_{i,j}^{\left(l\right)}.$
It is assumed throughout that conditional on $\mathcal{G}$, the $y_{i,j}^{\left(l\right)}$'s
are jointly independent across all $l$ and $i>j$. For ease of presentation,
we introduce the collection of sufficient
statistics as 
\begin{equation}
\boldsymbol{y}_{i}:=\left\{ y_{i,j}\mid j:\left(i,j\right)\in\mathcal{E}\right\}; \qquad y_{i,j}:=\frac{1}{L}\sum_{l=1}^{L}y_{i,j}^{\left(l\right)}.
\label{eq:sufficient_stats}
\end{equation}


\item \emph{Signal to Noise Ratio (SNR)~/~Quality of Comparisons}. The overall
faithfulness of the acquired evaluation between items $i$ and $j$
is captured by the sufficient statistic $y_{i,j}$.
Its SNR can be captured by 
\begin{align}
\mathsf{SNR}: & = \frac{\mathbb{E}^{2}\left[y_{i,j}\right]} {{{\bf Var}\left[y_{i,j}\right]}}~\asymp~L.\label{eq:defn-SNR}
\end{align}
As a result, the number $L$ of repeated comparisons measures the
SNR or the \emph{quality} of comparisons over any observed pair of items. 
\item \emph{Dynamic Range of Preference Scores}. It is assumed throughout
that the dynamic range of the preference scores is fixed irrespective
of $n$, namely, 
\begin{equation}
w_{i}\in\left[w_{\min},w_{\max}\right],\quad\quad1\leq i\leq n\label{eq:DynamicRange}
\end{equation}
for some positive constants $w_{\min}$ and $w_{\max}$ bounded
away from 0, which amounts to the most challenging regime \citep{Negahban2012}. 
In fact, the case in which the range ${w_{\max}}/{w_{\min}}$
grows with $n$ can be readily translated into the above fixed-range regime by first
separating out those items with vanishing scores (e.g. via a simple voting method
like Borda count \citep{ammar2011ranking}). 
 
\end{itemize}
\textbf{Performance Metric}. Given these pairwise observations,
one wishes to see whether or not the top-$K$ ranked items
are identifiable. To this end, we consider the probability of error $P_{\mathrm{e}}$ in isolating the {\em set} of top-$K$ ranked items, i.e.
\begin{align}
P_{\mathrm{e}}\left(\psi\right) & :=\mathbb{P}\Big\{\psi\left(\boldsymbol{y}\right)\neq[K]\Big\},\label{eq:prob-error}
\end{align}
where $\psi$
is any ranking scheme that returns a set of $K$ indices.
Here, $\left[K\right]$ denotes the (unordered)\emph{ }set
of the first $K$ indices. We aim to characterize the fundamental
\emph{admissible} \emph{region} of $(L,p_{{\rm obs}})$ where
reliable top-$K$ ranking is feasible, i.e. $P_{{\rm e}}$
can be vanishingly small as $n$ grows.

\section{Minimax Ranking Limits\label{sec:Main-Results}}

We explore the fundamental ranking limits from a minimax perspective,
which centers on the design of \emph{robust} ranking schemes that
guard against the worst case in probability of error. The most challenging
component of top-$K$ rank aggregation often hinges upon distinguishing
the two items near the decision boundary, i.e. the $K^{\text{th}}$ and $(K+1)^{\text{th}}$ ranked items.
Due to the random nature of the acquired finite-bit comparisons, the information
concerning their relative preference could be obliterated by noise,
unless their latent preference scores are sufficiently separated.
In light of this, we single out a preference separation measure as
follows 
\begin{align}
\Delta_{K} & :=\frac{w_{K}-w_{K+1}}{w_{\max}}.\label{eq:separation}
\end{align}
As will be seen, this measure plays a crucial role in determining
information integrity for top-$K$ identification.

To model random sampling and partial observations, we employ the Erd\H{o}s\textendash{}R\'enyi
random graph $\mathcal{G}\sim\mathcal{G}_{n,p_{\mathrm{obs}}}$. As already
noted by \citep{ford1957solution}, if the comparison
graph $\mathcal{G}$ is not connected, then there is absolutely no
basis to determine relative preferences between two disconnected components.
Therefore, a reasonable necessary condition that one would expect
is the connectivity of $\mathcal{G}$, which requires 
\begin{equation}
p_{\mathrm{obs}}> \frac{\log n}{n}.\label{eq:Connectivity}
\end{equation}
All results presented in this paper will operate under this assumption.

A main finding of this paper is an order-wise tight sufficient condition for top-$K$ identifiability, as stated in the theorem below.

\begin{theorem}[\textbf{Identifiability}] \label{thm:achievability}
Suppose that ${\cal G}\sim{\cal G}_{n,p_{{\rm obs}}}$ with $p_{{\rm obs}}\geq c_{0}\log n/n$. Assume that $L=O\left(\mathrm{poly}\left(n\right)\right)$
and ${w_{\max}}/{w_{\min}} = \Theta(1)$. With probability
exceeding $1-c_{1}n^{-2}$, the set of top-$K$ ranked items can be
identified exactly by an algorithm that runs in time $O\left(\left|\mathcal{E}\right|\log^{2}n\right)$,
provided that 
\begin{align}
L & \text{ }\geq\text{ }\frac{c_{2}\log n}{np_{{\rm obs}}\Delta_{K}^{2}}.\label{eq:achievability}
\end{align}
Here, $c_0,c_{1},c_{2}>0$ are some universal constants. 
\end{theorem}

\begin{remark}We assume throughout that the input fed to each ranking
algorithm is the sufficient statistic $\left\{ y_{i,j}\mid\left(i,j\right)\in\mathcal{E}\right\} $
rather than the entire collection of $y_{i,j}^{(l)}$,
otherwise even reading all data takes at least $O\left(L\cdot\left|\mathcal{E}\right|\right)$ flops~/~time.\end{remark}

Theorem \ref{thm:achievability} characterizes an {\em identifiable
region} within which exact identification of top-$K$ items is plausible
by nearly linear-time algorithms. The algorithm we propose, as detailed
in Section~\ref{sec:achievability}, attempts recovery by computing a score estimate whose errors can be uniformly controlled across all entries. 
Afterwards, the algorithm reports
the $K$ items that receive the highest estimated scores.

Encouragingly, the above identifiable region is minimax optimal. Consider
a given separation condition $\Delta_{K}$, and suppose that nature
behaves in an adversarial manner by choosing the worst-case scores
$\boldsymbol{w}$ compatible with $\Delta_{K}$. This imposes a minimax
lower bound on the quality of comparisons necessary for reliable ranking,
as given below.

\begin{theorem}[\textbf{Minimax Lower Bounds}] \label{thm:optimality}
Fix $\epsilon\in\left(0,\frac{1}{2}\right)$, and let ${\cal G}\sim{\cal G}_{n,p_{{\rm obs}}}$.
If 
\begin{align}
L & \text{ }\leq\text{ }c\frac{\left(1-\epsilon\right)\log n-2}{np_{\mathrm{obs}}\Delta_{K}^{2}}
\end{align}
holds for some constant%
\footnote{More precisely, one can take $c=w_{\min}^{4}/(2w_{\max}^{4})$.%
} $c>0$, then for any ranking scheme $\psi$, there exists a preference
vector $\boldsymbol{w}$ with separation $\Delta_{K}$ such that $P_{\mathrm{e}}\left(\psi\right)\geq\epsilon$.
\end{theorem} 


Theorem \ref{thm:optimality} taken collectively with Theorem \ref{thm:achievability}
determines the scaling of the fundamental ranking boundary on $L$.
Since the sample size sharply concentrates around $n^{2}p_{\mathrm{obs}}L$
in our model, this implies that the required sample complexity for
top-$K$ ranking scales inversely with the preference separation at
a quadratic rate. Put another way, Theorem \ref{thm:optimality}
justifies the need for a minimum separation criterion that applies to any ranking
scheme: 
\begin{equation}
\Delta_{K}\text{ }\gtrsim\text{ }\sqrt{ \frac{\log n}{np_{{\rm obs}}L}}.\label{eq:MinimumSeparation}
\end{equation}
Somewhat unexpectedly, there is no computational barrier away from this
statistical limit (at least in an order-wise sense). Several other remarks of Theorem \ref{thm:achievability} and Theorem \ref{thm:optimality}
are in order. 
\begin{itemize}[leftmargin=*, listparindent =1em, itemsep=0.1em]
\item \textbf{$\ell_{2}$ Loss vs. $\ell_{\infty}$ Loss}. A dominant fraction
of prior methods focus on the mean squared error in estimating the latent
scores $\boldsymbol{w}$. It was established by \citep{Negahban2012}
that the minimax $\ell_2$ regret is squeezed between 
\[
\frac{1}{\sqrt{np_{{\rm obs}}L}}\text{ }\lesssim\text{ }\inf_{\hat{\boldsymbol{w}}}\sup_{{\boldsymbol{w}}}\frac{\mathbb{E}\left[\|\hat{\boldsymbol{w}}-\boldsymbol{w}\|\right]}{\|\boldsymbol{w}\|}\text{ }\lesssim\text{ }\sqrt{\frac{\log n}{np_{{\rm obs}}L}},
\]
where the infimum is taken over all score estimators $\hat{\boldsymbol{w}}$.
This limit is almost identical to the minimax separation criterion (\ref{eq:MinimumSeparation})
we derive for top-$K$ identification, except for a potential logarithmic factor. In fact, if the pointwise error of  $\hat{\boldsymbol{w}}$ is {\em uniformly} bounded by $\sqrt{\log n/(np_{\mathrm{obs}}L)}$, then $\hat{\boldsymbol{w}}$ necessarily achieves the minimax $\ell_2$ error. Moreover, the pointwise error bound presents a fundamental bottleneck for top-$K$ ranking --- it will be impossible to differentiate the $K^{\text{th}}$ and $(K+1)^{\text{th}}$ ranked items unless their score separation exceeds the aggregate error of the corresponding score estimates for these two items.  Based on this observation, our algorithm is mainly designed to
control the elementwise estimation error. As will be seen in Section \ref{sec:achievability}, the resulting
estimation error will be uniformly spread over all entries, which
is optimal in both $\ell_{2}$ and $\ell_{\infty}$ sense. 

\item \textbf{From Coarse to Detailed Ranking}. The identifiable
region we present depends only on the preference separation between
items $K$ and $K+1$. This arises since we only intend to  identify the
group of top-$K$ items without specifying the fine details within
this group---we term it ``coarse top-$K$ ranking''. In fact, our results readily uncover the minimax separation
requirements for the case where one further expects fine ordering
among these $K$ items. Specifically, this task is feasible---in the
minimax sense---if and only if
\begin{equation}
\Delta_{i}\text{ }\gtrsim\text{ }\sqrt{ \frac{\log n} {{np_{\mathrm{obs}}L}}},\qquad1\leq i\leq K.
\label{eq:detailed-ranking}
\end{equation}
In words, the feasibility of detailed top-$K$ ranking relies on sufficient score separation between any consecutive pair of the top-$K$ ranked items.

\item \textbf{High SNR Requirement for Total Ordering}. In many situations,
the separation criterion (\ref{eq:detailed-ranking}) immediately suggests the hardness (or
even impossibility) of recovering the ordering over all
items. In fact, to figure out the total order, one expects sufficient
score separation between all pairs of consecutive items, namely,
\[
\Delta_{i}\text{ }\gtrsim\text{ }\sqrt{ \frac{\log n} {{np_{\mathrm{obs}}L}} },\qquad  1\leq i < n.
\]
Since the $\Delta_{i}$'s are defined in a normalized way (\ref{eq:separation}),
they necessarily satisfy 
\[
\sum_{i=1}^{n-1}\Delta_{i}=\frac{w_{1}-w_{n}}{w_{\max}}\leq1.
\]
As can be easily verified, the preceding two conditions would be incompatible
unless 
\begin{align*}
L & \text{ }\gtrsim\text{ } \frac{n\log n} {p_{\mathrm{obs}}},
\end{align*}
which imposes a fairly stringent SNR requirement. For instance, under
a sparse graph where $p_{\mathrm{obs}}\asymp\frac{\log n}{n}$, the
number of repeated comparisons (and hence the SNR) needs to be at
least on the order of $n^2$, regardless of the method employed. Such a
high SNR requirement could be increasingly more difficult to guarantee
as $n$ grows. 
\item \textbf{Passive Ranking vs.~Active Ranking}. In our passive ranking
model, the sample complexity requirement $n^{2}p_{\mathrm{obs}}L$
for reliable top-$K$ identification is given by 
\[
n^{2}p_{\mathrm{obs}}L ~\gtrsim~ \frac{n\log n}{\Delta_{K}^{2}}.
\]
In comparison, when adaptive sampling is employed for the preference-based
model, the most recent upper bound on the sample complexity (e.g.~Theorem 1 of
\citep{Hullermeier2013topKranking}) is on the order
of 
\[
\sum\nolimits_{i=1}^{n-1} \frac{1}{ \Delta_{i}^{2}}\log n.
\]
In the challenging regime where a dominant fraction of consecutive pairs are minimally
separated (e.g.~$\Delta_{1}=\cdots=\Delta_{n-1}$), the above results seem to suggest that active ranking may not outperform passive ranking, since the sample complexity reads $(n\log n) / \Delta_1^2$. For the other
extreme case where only a single pair is minimally separated (e.g.~$\Delta_{1}\ll\Delta_{i}$
($i\geq2$)), active ranking is more appealing,  because it will adaptively
acquire more paired evaluation over the minimally separated items instead of wasting samples on those pairs that are easy to differentiate. 
\end{itemize}

\section{Ranking Scheme: Spectral Method Meets MLE\label{sec:achievability}}

This section presents a nearly linear-time algorithm that attempts
recovery of the top-$K$ ranked items. The algorithm proceeds in two
stages: (1) an appropriate initialization that concentrates around
the ground truth in an $\ell_{2}$ sense, which can be obtained via
a spectral ranking method; (2) a sequence of iterative updates sharpening
the estimates in a pointwise manner, which consists in computing coordinate-wise
MLE solutions. The two stages operate upon
different sets of samples, while {\em no further sample splitting} is needed
within each stage. The combination of these two stages will be referred
to as \emph{Spectral MLE}.

Before continuing to describe the details of our algorithm, we introduce a few notations
that will be used throughout. 
\begin{itemize}
\item $\mathcal{L}\left(\boldsymbol{w};\boldsymbol{y}_{i}\right)$: the
likelihood function of a latent preference vector $\boldsymbol{w}$,
given the part of comparisons $\boldsymbol{y}_{i}$ that have bearing
on item $i$. 

\item $\boldsymbol{w}_{\backslash i}$: for any preference vector $\boldsymbol{w}$,
let $\boldsymbol{w}_{\backslash i}$ represent $\left[w_{1},\cdots,w_{i-1},w_{i+1},\cdots,w_{n}\right]$
excluding $w_{i}$. 

\item $\mathcal{L}\left(\tau,\boldsymbol{w}_{\backslash i};\boldsymbol{y}_{i}\right)$:
with a slight abuse of notation, denote by $\mathcal{L}\left(\tau,\boldsymbol{w}_{\backslash i};\boldsymbol{y}_{i}\right)$
the likelihood of the preference vector $\left[w_{1},\cdots,w_{i-1},\tau,w_{i+1},\cdots,w_{n}\right]$. 
\end{itemize}

\subsection{Algorithm: Spectral MLE}

It has been established that the spectral ranking method, particularly
\emph{Rank Centrality},
is able to discover a preference vector $\hat{\boldsymbol{w}}$ that
incurs minimal $\ell_{2}$ loss.
To enable reliable ranking, however, it is more desirable to obtain
an estimate that is faithful in an elementwise sense. Fortunately,
the solution returned by the spectral method will serve as an ideal
initial guess to seed our algorithm. The two components of the proposed
{Spectral MLE} are described below. 

\newsavebox\MLE 
\begin{lrbox}{\MLE}   
\begin{minipage}{\textwidth}    

\begin{equation}
w_{i}^{\mathrm{mle}}\leftarrow\arg\max_{\tau\in\left[w_{\min},w_{\max}\right]}\mathcal{L}\left(\tau,\boldsymbol{w}_{\backslash i}^{(t)};\text{ }\boldsymbol{y}_{i}^{\mathrm{iter}}\right).\label{eq:MLEiteration}
\end{equation}
\end{minipage} 
\end{lrbox}

\newsavebox\CheckGap 
\begin{lrbox}{\CheckGap}   
\begin{minipage}{\textwidth}    

\begin{equation}
w_{i}^{(t+1)}\leftarrow\begin{cases}
w_{i}^{\mathrm{mle}},\quad & \text{if }\left|w_{i}^{\mathrm{mle}}-w_{i}^{(t)}\right|>\xi_{t},\\
w_{i}^{(t)}, & \text{else}.
\end{cases}\label{eq:RefinementIteration}
\end{equation}
\end{minipage} 
\end{lrbox}

\begin{algorithm*}[t]
\caption{Spectral MLE.\label{alg:Nonconvex}}
\label{Algorithm:SpectralMLE}%
\begin{tabular}{>{\raggedright}p{1\textwidth}}
\textbf{Input}: The average comparison outcome $y_{i,j}$ for all
$(i,j)\in\mathcal{E}$; the score range $\left[w_{\min},w_{\max}\right]$.\vspace{0.7em}\tabularnewline
\textbf{Partition $\mathcal{E}$} randomly into two sets $\mathcal{E}^{\mathrm{init}}$
and $\mathcal{E}^{\mathrm{iter}}$ each containing $\frac{1}{2}\left|\mathcal{E}\right|$
edges. Denote by $\boldsymbol{y}_{i}^{\mathrm{init}}$ (resp. $\boldsymbol{y}_{i}^{\mathrm{iter}}$)
the components of $\boldsymbol{y}_{i}$ obtained over $\mathcal{E}^{\mathrm{init}}$
(resp. $\mathcal{E}^{\mathrm{iter}}$). \vspace{0.7em}\tabularnewline
\textbf{Initialize} $\boldsymbol{w}^{(0)}$ to be the estimate computed
by \emph{Rank Centrality} on $\boldsymbol{y}_{i}^{\mathrm{init}}$
($1\leq i\leq n$).\vspace{0.7em}\tabularnewline
\textbf{Successive Refinement: for $t=0:T$ do}\tabularnewline
\multirow{3}{1\textwidth}{$\quad$1) Compute the coordinate-wise MLE \usebox{\MLE} }\tabularnewline
\tabularnewline
\tabularnewline
$\quad$2) For each $1\leq i\leq n$, set \usebox{\CheckGap} \vspace{0.7em}\tabularnewline
\textbf{Output }the indices of the \textbf{$K$} largest components
of\textbf{ }$\boldsymbol{w}^{(T)}$.\tabularnewline
\end{tabular}
\end{algorithm*}

\begin{algorithm}
\caption{Rank Centrality \citep{Negahban2012}}
\label{Algorithm:RC}
\begin{tabular}{l}
\textbf{Input}: The average comparison outcome $y_{i,j}$ for all
$(i,j)\in\mathcal{E}^{\mathrm{iter}}$.\vspace{0.7em}\tabularnewline
\textbf{Compute} the transition matrix $\boldsymbol{P}=[P_{i,j}]_{1\leq i,j\leq n}$ such that\tabularnewline
$\qquad\qquad\qquad P_{i,j}=\begin{cases}
\frac{1}{d_{\max}} \frac{y_{j,i}}{y_{i,j}+y_{j,i}},\quad & \text{if }(i,j)\in\mathcal{E}^{\mathrm{iter}};\\
0, & \text{if }i\neq j\text{ and }(i,j)\notin\mathcal{E}^{\mathrm{iter}};\\
1-\frac{1}{d_{\max}}\sum_{k:(i,k)\in\mathcal{E}^{\mathrm{iter}}} \frac{y_{k,i}}{y_{i,k}+y_{k,i}},\quad & \text{if }i=j.
\end{cases}$ \vspace{0.3em}\tabularnewline
where $d_{\max}$ is the maximum out-degrees of vertices in $\mathcal{E}^{\mathrm{iter}}$. \vspace{0.7em}\tabularnewline 
\textbf{Output} the stationary distribution of
$\boldsymbol{P}$\tabularnewline
\end{tabular}
\end{algorithm}

\begin{enumerate}[leftmargin=*, listparindent =1em, itemsep=0.1em]
\item \textbf{Initialization via Spectral Ranking}. We generate an initialization
$\boldsymbol{w}^{(0)}$ via { Rank Centrality}. In words, Rank Centrality proceeds by constructing a Markov chain based on the
pairwise observations, and then returning its stationary distribution
by computing the leading eigenvector of the associated probability
transition matrix. For the sake of completeness, we provide the detailed procedure of Rank Centrality in Algorithm \ref{Algorithm:RC}. Under the Erd\H{o}s\textendash{}R\'enyi model,
the estimate $\boldsymbol{w}^{(0)}$ is known to be reasonably faithful in terms
of the mean squared loss \citep{Negahban2012}, that is, with high probability,
\begin{align*}
\frac{\|\boldsymbol{w}^{(0)}-\boldsymbol{w}\| }{ \left\Vert \boldsymbol{w}\right\Vert }  ~\lesssim~ \sqrt{ \frac{\log n} {np_{\mathrm{obs}}L}}.
\end{align*}

\item \textbf{Successive Refinement via Coordinate-wise MLE.} Note that
the state-of-the-art finite-sample analyses for MLE (e.g. \citep{Negahban2012}) involve
only the $\ell_{2}$ accuracy of the global MLE when the locations of all samples are i.i.d. (rather than the graph-based model considered herein). Instead of seeking
a global MLE solution, we propose to carefully utilize the coordinate-wise MLE. Specifically, we cyclically iterate through each component,
one at a time, maximizing the log-likelihood function with respect
to that component. In contrast to the coordinate-descent method
for solving the global MLE, we replace the preceding estimate with
the new coordinate-wise MLE only when  they are far apart.
Theorem \ref{thm:refinement}
(to be stated in Section 4.2) guarantees the contraction of the
pointwise error for each cycle, which leads to a geometric convergence
rate. 
\end{enumerate}
The algorithm then returns the indices of top-$K$ items in accordance
to the score estimate. A formal and detailed description of the
procedure is summarized in Algorithm \ref{Algorithm:SpectralMLE}.

\begin{remark} We split $\mathcal{E}$ into $\mathcal{E}^{\mathrm{init}}$ and $\mathcal{E}^{\mathrm{iter}}$ for analytical convenience. Empirically, if we keep $\mathcal{E}^{\mathrm{init}}=\mathcal{E}^{\mathrm{iter}}=\mathcal{E}$ and reuse all samples, then it seems to slightly outperform the sample-splitting procedure. Thus, we recommend the sample-reusing procedure for practical use, and leave the theoretical justification for future work. 
\end{remark}

\begin{remark}Spectral MLE is inspired by recent advances in solving
non-convex programs by means of iterative methods \citep{KMO:it,keshavan2009matrix,jain2013low, candes2014phase,netrapalli2013phase,balakrishnan2014statistical}.
A key message conveyed from these works is:
once we arrive at an appropriate initialization (often via a spectral
method), the iterative estimates will  be rapidly attracted towards
the global optimum.\end{remark}

\begin{remark}While our analysis is restricted to the Erd\H{o}s\textendash{}R\'enyi model, Spectral MLE is applicable to general graphs. We caution,
however, that spectral ranking is not guaranteed to achieve
minimal $\ell_{2}$ loss for general graphs and, in particular,  the kind of graphs exhibiting small
spectral gaps. Therefore, Spectral MLE is not necessarily minimax optimal under
general graph patterns. \end{remark}

Notably, the successive refinement stage is developed based on the observation that we are able to characterize the confidence intervals of the coordinate-wise MLEs at each iteration. The role of such confidence intervals is to help detect outlier components that incur large
pointwise loss. Since the initial guess is optimal in an overall $\ell_2$ sense, a large fraction of its entries are already faithful relative to the ground truth.  As a consequence, it suffices to disentangle the ``sparse'' set of  outliers.

One appealing feature of Spectral MLE is its low computational complexity.
Recall that the initialization step by Rank Centrality can be solved
for $\epsilon$ accuracy---i.e. identifying an estimate $\hat{\boldsymbol{w}}$ such that $\|\hat{\boldsymbol{w}}-\boldsymbol{w}\|/\|\boldsymbol{w}\|\leq \epsilon$---within $O\left(\left|\mathcal{E}\right|\log\frac{1}{\epsilon}\right)$
time instances by means of a power method. In addition, for each component $i$, the coordinate-wise likelihood function
involves a sum of $\mathrm{deg}\left(i\right)$ terms. Since finding
the coordinate-wise MLE (\ref{eq:MLEiteration}) can be cast as a
one-dimensional convex program, one can get $\epsilon$ accuracy via
a bisection method within $O\left(\mathrm{deg}\left(i\right)\cdot\log\frac{1}{\epsilon}\right)$
time. Therefore, each iteration cycle of the successive
refinement stage can be accomplished in time $O\left(\left|\mathcal{E}\right|\cdot\log\frac{1}{\epsilon}\right)$.

The following theorem establishes the ranking accuracy of Spectral
MLE under the BTL model.

\begin{theorem}\label{thm:SpectralMLE}Let $c_{0},c_1, c_2, c_{3}>0$
be some universal constants. Suppose that $L=O(\mathrm{poly}(n))$, the comparison
graph $\mathcal{G}\sim\mathcal{G}_{n,p_{\mathrm{obs}}}$ with $p_{\mathrm{obs}}>c_{0}\log n/n$, 
and  the separation measure (\ref{eq:separation}) satisfies 
\begin{equation}
\Delta_{K}\text{ }> \text{ }c_{1}\sqrt{ \frac{\log n}  {{np_{\mathrm{obs}}L}}}.\label{eq:OptimalGap}
\end{equation}
Then with probability exceeding $1-1/n^{2}$, Spectral MLE perfectly
identifies the set of  top-$K$ ranked items, provided that the algorithmic
parameters obey 
%
%
$T\geq c_{2}\log n$ and
\begin{equation}
\xi_{t}:=c_{3}\left\{ \xi_{\min}+\frac{1}{2^{t}}\left(\xi_{\max}-\xi_{\min}\right)\right\},
\label{eq:Delta_t_parameter}
\end{equation}
where $\xi_{\min}:=\sqrt{\frac{\log n}{np_{\mathrm{obs}}L}}$
and $\xi_{\max}:=\sqrt{\frac{\log n}{p_{\mathrm{obs}}L}}$.
\end{theorem}

Theorem \ref{thm:SpectralMLE} basically implies that the proposed
algorithm succeeds in separating out the high-ranking objects with
probability approaching one, as long as the preference score satisfies the separation
condition 
\[
\Delta_{K}\text{ }\gtrsim\text{ }\sqrt{\frac{\log n}{np_{\mathrm{obs}}L}}.
\]
Additionally, Theorem \ref{thm:SpectralMLE} asserts that the
number of iteration cycles required in the second stage scales at
most logarithmically, revealing that Spectral MLE achieves the desired
ranking precision with nearly linear-time computational complexity.

\subsection{Successive Refinement: Convergence and Contraction of $\ell_{\infty}$
Error\label{sub:Contractivity}}

In the sequel, we would like to provide some interpretation
as to why we expect the
pointwise error of the score estimates to be controllable. The argument is heuristic in nature, since we will assume
for simplicity that each iteration employs a fresh set of samples
$\boldsymbol{y}$ independent of the present estimate $\boldsymbol{w}^{(t)}$.

Denote by $\ell^{*}\left(\tau\right)$ the true log-likelihood function
\begin{equation}
\ell^{*}\left(\tau\right):=\frac{1}{L}\log\mathcal{L}\left(\tau,\boldsymbol{w}_{\backslash i};\boldsymbol{y}_{i}\right).\label{eq:true-ll}
\end{equation}
Straightforward calculation  suggests that its expectation around
$w_{i}$ can be controlled through a locally strongly-concave function, due to the existence of a second-order lower bound 
\begin{align}
\mathbb{E}_{\boldsymbol{w}}\left[\mathcal{\ell}^{*}\left(w_{i}\right)-\mathcal{\ell}^{*}\left(\tau\right)\right] 
 & ~=~ \sum_{j:(i,j)\in \mathcal{E}} \mathsf{KL}\left( \frac{w_i}{w_i+w_j} ~\Big\Vert~ \frac{\tau}{\tau+w_j}  \right) \nonumber \\
 & ~\gtrsim~ \left|\tau-w_{i}\right|^{2}np_{\mathrm{obs}},
\label{eq:penalty}
\end{align}
where $\mathsf{KL}(p~\Vert~q)$ represents the Kullback–Leibler (KL) divergence between $\mathsf{Bernoulli}(p)$ and $\mathsf{Bernoulli}(q)$.  These calculations will be made precise in Appendix \ref{sub:achievability-proof} (and in particular Eqn. (\ref{eq:LB_KL_E}) and (\ref{eq:PinskerLB})).

This measures the penalty when $\tau$ deviates from the ground truth.
Note, however, that we don't have direct access to  $\ell^{*}\left(\cdot\right)$
since it relies on the latent scores $\boldsymbol{w}$. To obtain a computable surrogate, we replace $\boldsymbol{w}$ with the present estimate
$\boldsymbol{w}^{(t)}$, resulting in the plug-in likelihood function
\[
\hat{\ell}_{i}\left(\tau\right):=\frac{1}{L}\log\mathcal{L}\left(\tau,\boldsymbol{w}_{\backslash i}^{(t)};\boldsymbol{y}_{i}\right).
\]
Fortunately, the surrogate loss incurred by employing $\hat{\ell}_{i}\left(\tau\right)$
is locally stable in the sense that, 
\begin{align}
 \left|\mathbb{E}_{\boldsymbol{w}}\left[\hat{\ell}_{i}\left(\tau\right)-\hat{\ell}_{i}\left(w_{i}\right)-\left(\mathcal{\ell}^{*}\left(\tau\right)-\mathcal{\ell}^{*}\left(w_{i}\right)\right)\right]\right| 
~ \lesssim ~
 np_{\mathrm{obs}}\left|\tau-w_{i}\right|\frac{\left\Vert \hat{\boldsymbol{w}}-\boldsymbol{w}\right\Vert }{\left\Vert \boldsymbol{w}\right\Vert },
\label{eq:estimation error}
\end{align}
which will be made clear in Appendix \ref{sub:achievability-proof}. This essentially means that the surrogate loss $\hat{\ell}_{i}\left(\tau\right)-\hat{\ell}_{i}\left(w_{i}\right)$ is a reasonably good approximation of the true loss $\mathcal{\ell}^{*}\left(\tau\right)-\mathcal{\ell}^{*}\left(w_{i}\right)$, as long as $\tau$ (resp.~$\hat{\boldsymbol{w}}$) is sufficiently close to $w_i$ (resp.~$\boldsymbol{w}$). 
As a result, any candidate $\tau\neq w_{i}$ will be viewed as \emph{less likely} than
and hence distinguishable from the ground truth $w_{i}$ (i.e. $\hat{\ell}(w_i)>\hat{\ell}(\tau)$)  in the mean sense, provided
that its deviation penalty (\ref{eq:penalty}) dominates the surrogate
loss (\ref{eq:estimation error}). This would hold as long as the pointwise loss exceeds the normalized $\ell_2$ loss:
\[
\left|\tau-w_{i}\right| ~\gtrsim~ \frac{\left\Vert \hat{\boldsymbol{w}}-\boldsymbol{w}\right\Vert}{\left\Vert \boldsymbol{w}\right\Vert }.
\]
Thus,  our procedure is expected to be able to converge to a solution
whose pointwise error is as low as the normalized $\ell_{2}$ error
of the initial guess.

Encouragingly, the $\ell_{\infty}$ estimation error not only converges,
but converges at a geometric rate as well. This rapid convergence
property does not rely on the ``fresh-sample'' assumption imposed
in the above heuristic argument, as formally stated in the following
theorem.

\begin{theorem}\label{thm:refinement}Suppose that $\mathcal{G}\sim\mathcal{G}_{n,p_{\mathrm{obs}}}$
with $p_{\mathrm{obs}}>c_{0}\log n/n$ for some large constant $c_{0}$,
and there exists a score vector $\hat{\boldsymbol{w}}^{\mathrm{ub}}\in\left[w_{\min},w_{\max}\right]^{n}$ independent of $\mathcal{G}$
satisfying 
\begin{align}
 \left|\hat{w}_{i}^{\mathrm{ub}}-w_{i}\right|\hspace{0.3em} & \leq\hspace{0.3em}\xi w_{\max},\quad 1\leq i \leq n;\label{eq:weight-boundedness}\\
\|\hat{\boldsymbol{w}}^{\mathrm{ub}}-\boldsymbol{w}\| & \leq\hspace{0.3em}\delta\left\Vert \boldsymbol{w}\right\Vert .\label{eq:goodness-of-estimate}
\end{align}
 Then with probability at least $1- {c_{1}}{n^{-4}}$ for some constant
$c_{1}>0$, 
the coordinate-wise MLE 
\begin{equation}
w_{i}^{\mathrm{mle}}:=\arg\max_{\tau\in\left[w_{\min},w_{\max}\right]}\mathcal{L}\left(\tau,\hat{\boldsymbol{w}}_{\backslash i};\boldsymbol{y}_{i}\right)\label{eq:MLE_defn}
\end{equation}
satisfies 
\begin{align}
\left|w_{i}-w_{i}^{\mathrm{mle}}\right|  &<\frac{20\left(6+ \frac{\log L}{\log n}\right)w_{\max}^{5}}{w_{\min}^{4}}    \cdot \max \left\{ \delta+\frac{\log n}{np_{\mathrm{obs}}}\cdot\xi,\text{ } \sqrt{\frac{\log n}{np_{\mathrm{obs}}L}}\right\} 
\label{eq:MLEgap}
\end{align}
simultaneously for all scores $\hat{\boldsymbol{w}} \in [w_{\min}, w_{\max}]^n$ obeying 
$\left|\hat{w}_{i}-w_{i}\right|\leq\left|\hat{w}_{i}^{\mathrm{ub}}-w_{i}\right|$, $1\leq i\leq n$.
\end{theorem}

In the regime where $L=O\left(\mathrm{poly}\left(n\right)\right)$ and $\frac{\|{\boldsymbol{w}^{(t)}}-\boldsymbol{w}\|}{\|\boldsymbol{w}\|} \asymp \delta\asymp \sqrt{\frac{\log n}{np_{\mathrm{obs}}L}}$,
Theorem \ref{thm:refinement} asserts that under appropriate conditions,
the coordinate-wise MLE $\boldsymbol{w}^{\mathrm{mle}}$ is expected to achieve a lower pointwise error than  $\boldsymbol{w}^{(t)}$ such that
\begin{equation}
\big\|\boldsymbol{w}^{\mathrm{mle}}-\boldsymbol{w}\big\|_{\infty}\lesssim\frac{\big\|\boldsymbol{w}^{(t)}-\boldsymbol{w}\big\|}{\left\Vert \boldsymbol{w}\right\Vert }+\frac{\log n}{np_{\mathrm{obs}}}\big\|\boldsymbol{w}^{(t)}-\boldsymbol{w}\big\|_{\infty}.
\label{eq:ShrinkRateMLE}
\end{equation}
When the replacement threshold $\xi_t$ is chosen to be on the same order as $\|\boldsymbol{w}^{\mathrm{mle}}-\boldsymbol{w}\big\|_{\infty}$, one can detect outliers and  drag down the elementwise
estimation error at a rate
\begin{equation}
\big\|\boldsymbol{w}^{(t+1)}-\boldsymbol{w}\big\|_{\infty}\lesssim\frac{\big\|\boldsymbol{w}^{(t)}-\boldsymbol{w}\big\|}{\left\Vert \boldsymbol{w}\right\Vert }+\frac{\log n}{np_{\mathrm{obs}}}\big\|\boldsymbol{w}^{(t)}-\boldsymbol{w}\big\|_{\infty}.
\label{eq:ShrinkRate}
\end{equation}
One important feature is that the same collection of samples can be \emph{reused
}across all iterations at the successive refinement stage, provided that we can identify in each cycle another slightly looser estimate that
is independent of the samples. 
Another property that will be made clear in the analysis is that the $\ell_2$ estimation error obeys
\begin{align}
	\qquad\qquad \frac{ \| \boldsymbol{w}^{(t)}  - \boldsymbol{w}\|}{\| \boldsymbol{w} \|} ~\lesssim~ \frac{ \| \boldsymbol{w}^{(0)}  - \boldsymbol{w}\|}{\| \boldsymbol{w} \|}
	~\lesssim~  \sqrt{\frac{\log n}{np_{\mathrm{obs}}L}},  \label{eq:L2_wt}  
\end{align}
which further gives
\begin{align}
	\big\|\boldsymbol{w}^{(t+1)}-\boldsymbol{w}\big\|_{\infty} ~\lesssim~ \sqrt{\frac{\log n}{np_{\mathrm{obs}}L}} +\frac{\log n}{np_{\mathrm{obs}}}\big\|\boldsymbol{w}^{(t)}-\boldsymbol{w}\big\|_{\infty}.
\label{eq:ShrinkRate-clean}
\end{align}
We recognize that the non-negative recursive sequence $\{f_n\}$ obeying the recurrence equation $f_n = a + bf_{n-1}$ ($0<b < 1$) must converge to a point\footnote{To see this, one can rewrite the recurrence inequality as $f_n - \frac{a}{1-b} = b( f_{n-1} - \frac{a}{1-b})$, which gives $f_{n}- \frac{a}{1-b} = b^n (f_0- \frac{a}{1-b})$.  When $n$ tends to infinity, this gives $f_n - \frac{a}{1-b} = 0$. } $f_{\infty} = \frac{a}{1-b}$.  When specialized to (\ref{eq:ShrinkRate-clean}), this fact implies that the output of Spectral MLE obeys
\begin{align*}
\big\|\boldsymbol{w}^{(T)} -\boldsymbol{w}\big\|_{\infty} ~\lesssim~   \sqrt{\frac{\log n}{np_{\mathrm{obs}}L}},
\end{align*}
as long as ${\log n}/{(np_{\mathrm{obs}})}$ is sufficiently small and $T$ is sufficiently large. This is minimally apart from the ground truth.

\subsection{Discussion\label{sub:Discussion-SpectralMLE}}

\begin{figure*}[t]
\begin{centering}
\begin{tabular}{cc}
\includegraphics[width=0.47\textwidth]{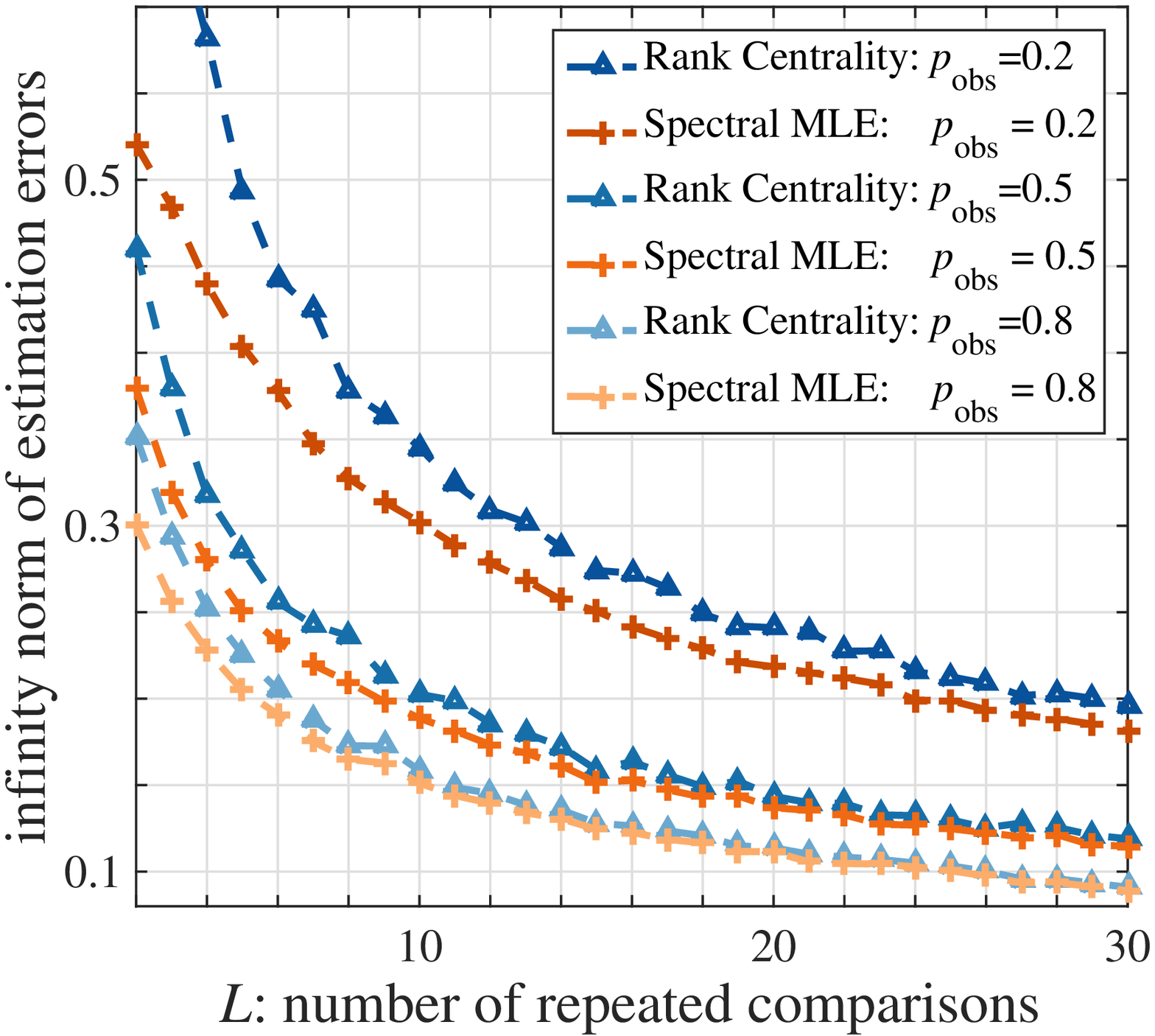} & \includegraphics[width=0.47\textwidth]{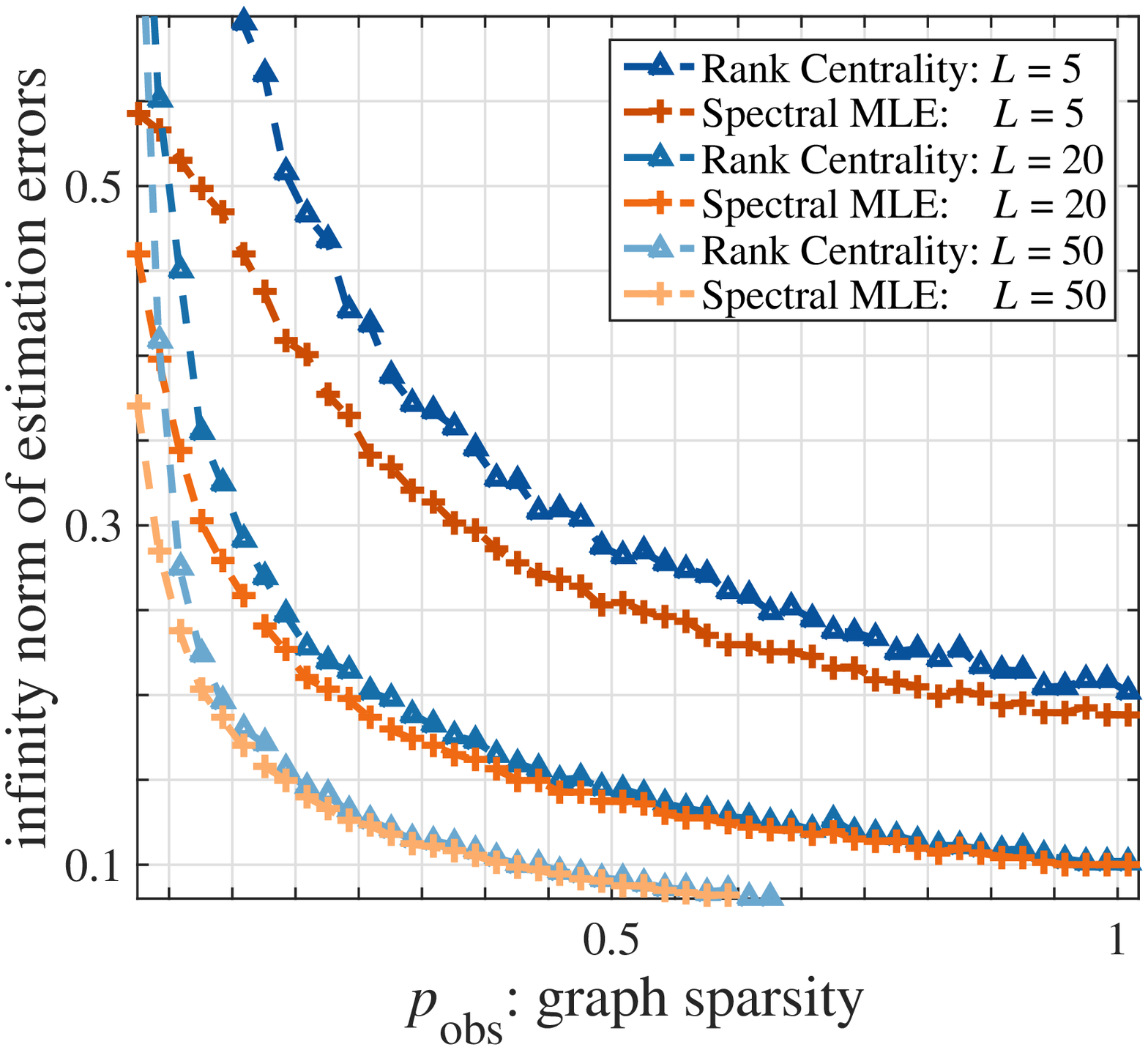}   \tabularnewline
(a) & (b) 
\tabularnewline
\end{tabular}
\begin{tabular}{c}
\includegraphics[width=0.47\textwidth]{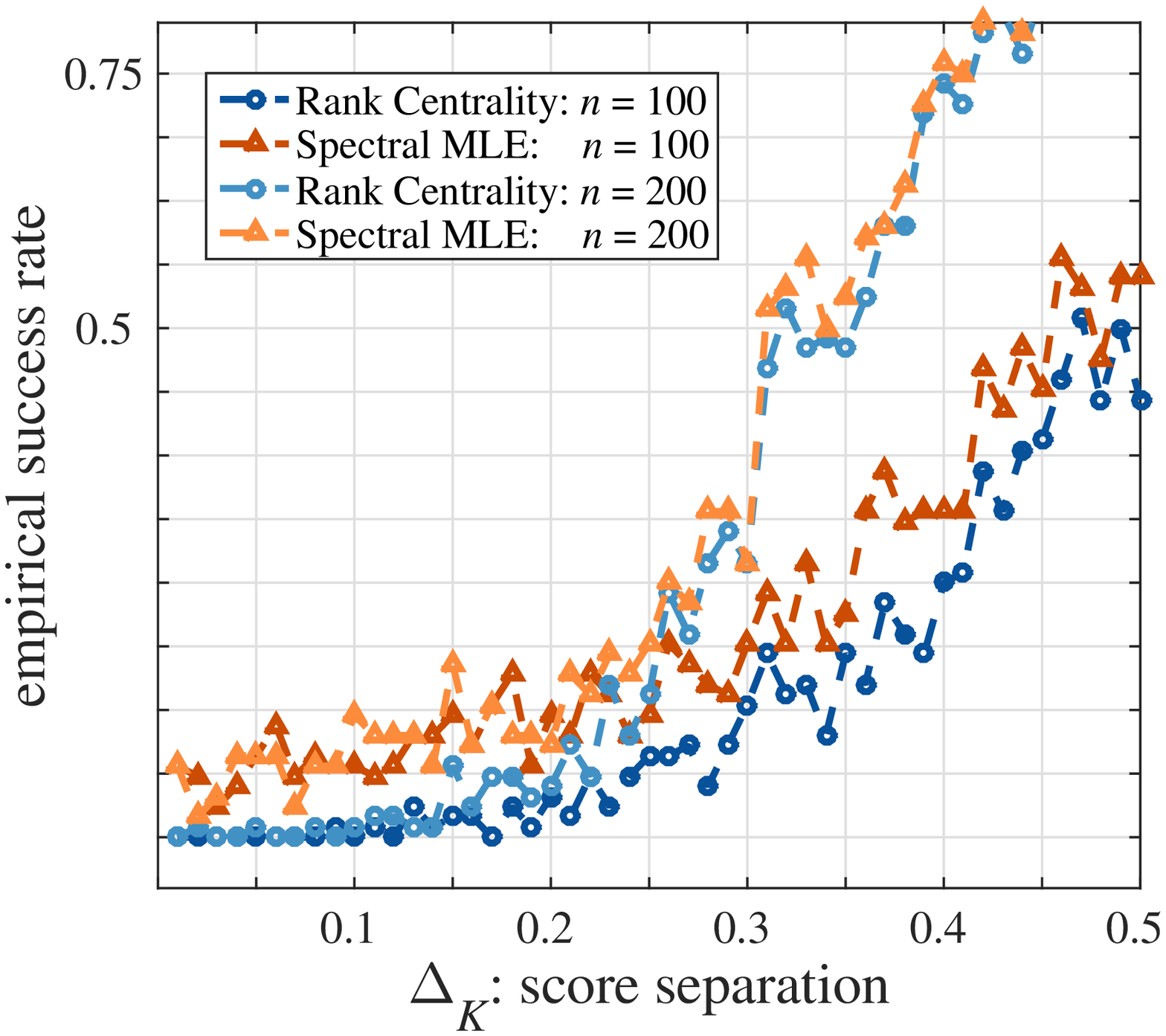} \tabularnewline
(c) \tabularnewline
\end{tabular}
\par\end{centering}

\caption{(a) Empirical $\ell_{\infty}$ loss v.s. $L$; (b) $\ell_{\infty}$
loss v.s. $p_{\mathrm{obs}}$; (c) Rate of success in top-$K$ identification ($n=100,200$).}
\label{figure:Numerical}
\end{figure*}

\textbf{Choice of Initialization}. Careful readers will remark that
the success of Spectral MLE can be guaranteed by a broader
selection of initialization procedures beyond Rank Centrality. Indeed, Theorem
\ref{thm:refinement} and subsequent analyses lead to the following
assertion: as long as the initialization method is able to produce an initial estimate
$\boldsymbol{w}^{(0)}$ that is reasonably faithful in the $\ell_2$ sense 
\begin{equation}
	\frac{\|\boldsymbol{w}^{(0)}-\boldsymbol{w}\| }{ {\|\boldsymbol{w}\|} } ~\lesssim~ \sqrt{ \frac{\log n} {np_{\mathrm{obs}}L}},
\label{eq:Initialization}
\end{equation}
then Spectral MLE will converge to a pointwise optimal preference
$\boldsymbol{w}^{(T)}$ obeying 
\[
\|\boldsymbol{w}^{(T)}-\boldsymbol{w}\|_{\infty} ~\lesssim~ \sqrt{\frac{\log n}{np_{\mathrm{obs}}L}}.
\]

\textbf{Initialization via Global MLE}. One would naturally wonder whether we can employ the global MLE (computed over $\boldsymbol{y}^{\mathrm{init}}$) to seed the iterative refinement stage (applied over $\boldsymbol{y}^{\mathrm{iter}}$). 
In fact, the state-of-the-art analysis (with a different but order-wise equivalent
model) \citep{Negahban2012} asserts that the global MLE
satisfies the desired $\ell_{2}$ property (\ref{eq:Initialization})
for at least two cases:  (a) complete
graphs, i.e. $p_{\mathrm{obs}}=1$, and (b) $ $Erd\H{o}s\textendash{}R\'enyi
graphs with no repeated comparisons, i.e. $L=1$. 
In these two cases, the proposed algorithm achieves minimal $\ell_{\infty}$ errors if we initialize it via the global MLE.

Nevertheless, whether the global MLE achieves minimal $\ell_{2}$
loss for other configurations $\left(L,p_{\mathrm{obs}}\right)$
has not been established. The analytical bottleneck seems to stem
from an underlying bias-variance tradeoff when accounting for
two successive randomness mechanisms: the random graph $\mathcal{G}$
and the repeated comparisons generated over $\mathcal{G}$. In general,
 $y_{i,j}^{(l)}$'s are not jointly independent
 unless we condition on $\mathcal{G}$. In
contrast, the above two special cases amount to two extreme situations:
(a) the randomness of $\mathcal{G}$ goes away when $p_{\mathrm{obs}}=1$;
(b) the condition $L=1$ avoids repeated sampling. Nevertheless, these two cases alone (as well
as the model in Theorem 4 of \citep{Negahban2012}) are not sufficient
in characterizing the complete tradeoff between graph sparsity and
the quality of the acquired comparisons.

\subsection{Numerical Experiments\label{sub:Numerical-Experiments}}

A series of synthetic experiments is conducted to demonstrate the
practical applicability of Spectral MLE. The important implementation
parameters in our approach is the choice of $c_{2}$ and $c_{3}$ given in Theorem
\ref{thm:SpectralMLE}, which specify $T$ and $\xi_{t}$, respectively. In
all numerical simulations performed here, we  pick $c_{2}=5$
and $c_{3}=1$, and do not split samples. 
We focus on the case where $n=100$, where each reported result is calculated by
averaging over 200 Monte Carlo trials.

We first examine the $\ell_{\infty}$ error of the score estimates. The latent scores are generated
uniformly over $\left[0.5,1\right]$. For each $\left(p_{\mathrm{obs}},L\right)$,
the paired comparisons are randomly generated as per the BTL model,
and we perform score inference by means of both Rank Centrality and
Spectral MLE.  Fig. \ref{figure:Numerical}(a)
(resp. Fig. \ref{figure:Numerical}(b)) illustrates the empirical tradeoff
between the pointwise score estimation accuracy and the number $L$
of repeated comparisons (resp. graph sparsity $p_{\mathrm{obs}}$).
It can be seen from these plots that the proposed Spectral MLE outperforms
Rank Centrality uniformly over all configurations, corroborating
our theoretical results. Interestingly, the performance gain is the most
significant under sparse graphs in the presence of low-resolution comparisons (i.e.
when $p_{\mathrm{obs}}$ and $L$ are small).

Next, we study the success
rate of top-$K$ identification as the number $n$ of items
varies. We generate the latent scores randomly over $\left[0.5,1\right]$, except that a separation $\Delta_K$ is imposed between items $K$ and $K+1$. The results are shown in Fig. \ref{figure:Numerical}(c) for the case where $p_{\mathrm{obs}}=0.2$,
and $L=5$. As can be seen, Spectral MLE achieves higher ranking accuracy compared to Rank Centrality for all these situations. Interestingly, the benefit of Spectral MLE relative to Rank Centrality is more apparent in the regime where the score separation is small. In addition, it seems that Rank Centrality is capable of achieving good ranking accuracy in the randomized model we simulate, and we leave the theoretical analysis for future work.

\section{Conclusion\label{sec:Conclusion}}

This paper investigates rank aggregation from pairwise data that emphasizes
the top-$K$ items. We developed a nearly linear-time algorithm that performs as well  as the
best model aware paradigm, from a minimax perspective.
 The proposed algorithm returns the indices of the best-$K$ items in accordance
 to a carefully tuned preference score estimate, which
 is obtained by combining a spectral method and a coordinate-wise MLE.
Our results uncover the identifiability limit of top-$K$ ranking,
which is dictated by the preference separation between the
$K^{\text{th}}$ and $(K+1)^{\text{th}}$  items.

This paper comes with some limitations in developing tight sample
complexity bounds under general graphs. The performances of Spectral MLE under other sampling models are worth investigating \citep{osting2015analysis}. 
In addition, it remains to characterize both statistical and computational 
ranking
limits for other choice models (e.g. the Plackett-Luce model \citep{hajek2014minimax}). It would also be interesting to consider the case where the paired comparisons are drawn from a mixture of BTL models (e.g. \citep{oh2014learning}), as well as the collaborative ranking  setting where one aggregates the item preferences from a pool of different users in order to infer rankings for each individual user (e.g. \citep{lu2014individualized,park2015preference}).


\acks{C. Suh was partly supported by the ICT R\&D program of MSIP/IITP [B0101-15-1272]. 
Y. Chen would like to thank Yuejie Chi for helpful discussion.  }


\newpage

\appendix

\section{Performance Guarantees for Spectral MLE\label{sec:Proof-SpectralMLE}}

In this section, we establish the theoretical guarantees of Spectral
MLE in controlling the ranking accuracy and $\ell_{\infty}$ estimation
errors, which are the subjects of Theorem \ref{thm:SpectralMLE} and
Theorem \ref{thm:refinement}. The proof of Theorem \ref{thm:SpectralMLE}
relies heavily on the claim of Theorem \ref{thm:refinement}; for
this reason, we present the proofs of Theorem \ref{thm:SpectralMLE}
and Theorem \ref{thm:refinement} in a reverse order. Before proceeding,
we recall that the coordinate-wise log-likelihood of $\tau$ is given
by 
\begin{equation}
\frac{1}{L}\log\mathcal{L}\left(\tau,\hat{\boldsymbol{w}}_{\backslash i};\boldsymbol{y}_{i}\right):=\sum_{j:(i,j)\in\mathcal{E}}y_{i,j}\log\frac{\tau}{\tau+\hat{w}_{j}}+\left(1-y_{i,j}\right)\log\frac{\hat{w}_{j}}{\tau+\hat{w}_{j}},\label{eq:defn-LL}
\end{equation}
and we shall use $\boldsymbol{w}_{\backslash i}$ (resp. $\hat{\boldsymbol{w}}_{\backslash i}$)
to denote the vector $\boldsymbol{w}=\left[w_{1},\cdots,w_{n}\right]$
(resp. $\hat{\boldsymbol{w}}=\left[\hat{w}_{1},\cdots,\hat{w}_{n}\right]$)
excluding the entry $w_{i}$ (resp. $\hat{w}_{i}$).

\subsection{Proof of Theorem \ref{thm:refinement}\label{sub:achievability-proof}}

To prove Theorem \ref{thm:refinement}, we need to demonstrate that
for every $\tau\in\left[w_{\min},w_{\max}\right]$ that is sufficiently
separated from  $w_{i}$ (or, more formally, $\left|\tau-w_{i}\right|\gtrsim\max\left\{ \delta+\frac{\xi\log n}{np_{\mathrm{obs}}},\text{ }\sqrt{\frac{\log n}{np_{\mathrm{obs}}L}}\right\} $),
the coordinate-wise likelihood satisfies
\begin{equation}
\log\mathcal{L}\left(w_{i},\hat{\boldsymbol{w}}_{\backslash i};\boldsymbol{y}_{i}\right)>\log\mathcal{L}\left(\tau,\hat{\boldsymbol{w}}_{\backslash i};\boldsymbol{y}_{i}\right)\label{eq:likelihood-compare}
\end{equation}
and, therefore, $\tau$ cannot be the coordinate-wise MLE. 

To begin with, we provide a lemma (which will be proved later) that
concerns (\ref{eq:likelihood-compare}) for \emph{any single} $\tau$
that is well separated from $w_{i}$.

\begin{lemma}\label{lemma-PointwiseError}Fix any $\gamma\geq3$.
Under the conditions of Theorem \ref{thm:refinement}, for any $\tau\in\left[w_{\min},w_{\max}\right]$
obeying 
\begin{equation}
\left|w_{i}-\tau\right|>\gamma\cdot\frac{w_{\max}^{5}}{w_{\min}^{4}}\max\left\{ \frac{25}{4}\left(\delta+\frac{\xi\log n}{np_{\mathrm{obs}}}\right),\text{ }\text{ }20\sqrt{\frac{\log n}{np_{\mathrm{obs}}L}}\right\} ,\label{eq:MLEdev}
\end{equation}
one has 
\begin{equation}
\frac{1}{L}\log\mathcal{L}\left(w_{i},\hat{\boldsymbol{w}}_{\backslash i};\boldsymbol{y}_{i}\right)-\frac{1}{L}\log\mathcal{L}\left(\tau,\hat{\boldsymbol{w}}_{\backslash i};\boldsymbol{y}_{i}\right)>\frac{w_{\max}^{6}}{100w_{\min}^{6}}\frac{\log n}{L}.\label{eq:Likelihood}
\end{equation}
with probability exceeding $1-6n^{-\gamma}$; this holds
simultaneously for all $\hat{\boldsymbol{w}}_{i}\in\left[w_{\min},w_{\max}\right]^{n}$
satisfying $|\hat{w}_i-w_i|\leq | \hat{w}_i^{\mathrm{ub}} -w_i|$, $1\leq i\leq n$. \end{lemma}

To establish Theorem \ref{thm:refinement}, we still need to derive
a uniform control over all $\tau$ satisfying (\ref{eq:MLEdev}).
This will be accomplished via a standard covering argument. Specifically,
for any small quantity $\epsilon>0$, we construct a set $\mathcal{N}_{\epsilon}$
(called an $\epsilon$-cover) within the interval $\left[w_{\min},w_{\max}\right]$
such that for any $\tau\in\left[w_{\min},w_{\max}\right]$, there
exists an $\tau_{0}\in\mathcal{N}_{\epsilon}$ obeying 
\begin{equation}
\left|\tau-\tau_{0}\right|\leq\epsilon\quad\text{and}\quad|\tau_{0}-w_{i}|\geq|\tau-w_{i}|.
\end{equation}
It is self-evident that one can produce such a cover $\mathcal{N}_{\epsilon}$
with cardinality $\left\lceil \frac{w_{\max}}{\epsilon}\right\rceil +1$.
If we set $\gamma=6+\frac{\log L}{\log n}$ in Lemma \ref{lemma-PointwiseError} (which obeys $\gamma=\Theta(1)$ since $L=O(\mathrm{poly}(n))$),
taking the union bound over $\mathcal{N}_{\epsilon}$ gives 
\begin{equation}
\frac{1}{L}\log\mathcal{L}\left(w_{i},\hat{\boldsymbol{w}}_{\backslash i};\boldsymbol{y}_{i}\right)-\frac{1}{L}\log\mathcal{L}\left(\tau_{0},\hat{\boldsymbol{w}}_{\backslash i};\boldsymbol{y}_{i}\right)>\frac{w_{\max}^{6}}{100w_{\min}^{6}}\frac{\log n}{L}\label{eq:discretize}
\end{equation}
simultaneously over all $\tau_{0}\in\mathcal{N}_{\epsilon}$ obeying
\[
\left|w_{i}-\tau_{0}\right| ~>~ \frac{\left(6+\frac{\log L}{\log n}\right)w_{\max}^{5}}{w_{\min}^{4}}\max\left\{ \frac{25}{4}\left(\delta+\frac{\xi\log n}{np_{\mathrm{obs}}}\right),\text{ }\text{ }20\sqrt{\frac{\log n}{np_{\mathrm{obs}}L}}\right\} ;
\]
this occurs with probability at least $1-6\left|\mathcal{N}_{\epsilon}\right|n^{-6-\frac{\log L}{\log n}}$.

We proceed by bounding the difference between $\log\mathcal{L}\left(\tau,\hat{\boldsymbol{w}}_{\backslash i};\boldsymbol{y}_{i}\right)$
and $\log\mathcal{L}\left(\tau_{0},\hat{\boldsymbol{w}}_{\backslash i};\boldsymbol{y}_{i}\right)$ for any $|\tau-\tau_0| \leq \epsilon$.
To achieve this, we first recognize that the Lipschitz constant of
$\frac{1}{L}\log\mathcal{L}\left(\tau,\hat{\boldsymbol{w}}_{\backslash i};\boldsymbol{y}_{i}\right)$ 
(cf. (\ref{eq:defn-LL})) w.r.t.~$\tau$ is bounded above by the following inequality:
\begin{align*}
\frac{1}{L}\cdot\left|\frac{\partial\log\mathcal{L}\left(\tau,\hat{\boldsymbol{w}}_{\backslash i};\boldsymbol{y}_{i}\right)}{\partial\tau}\right| & ~=~ \left|\sum_{j:(i,j)\in\mathcal{E}}y_{i,j}\left(\frac{1}{\tau}-\frac{1}{\tau+\hat{w}_{j}}\right)-\left(1-y_{i,j}\right)\frac{1}{\tau+\hat{w}_{j}}\right|\\
 & \overset{(\text{a})}{\leq}~ \mathrm{deg}\left(i\right)\cdot\frac{2}{w_{\min}} 
~\overset{(\text{b})}{\leq}~ \frac{12}{5}\frac{np_{\mathrm{obs}}}{w_{\min}},
\end{align*}
where (a) follows since 
\begin{align*}
& \left|y_{i,j}\left(\frac{1}{\tau}-\frac{1}{\tau+\hat{w}_{j}}\right)-\left(1-y_{i,j}\right)\frac{1}{\tau+\hat{w}_{j}}\right|
=\left|\frac{y_{i,j}}{\tau}-\frac{1}{\tau+\hat{w}_{j}}\right|  \\
& \quad ~\leq~\left|\frac{y_{i,j}}{\tau}\right|+\left|\frac{1}{\tau+\hat{w}_{j}}\right| 
~<~\frac{2}{w_{\min}},
\end{align*}
and (b) holds since $\mathrm{deg}(i)\leq2.4np_{\mathrm{obs}}$ with
probability $1-O\left(n^{-10}\right)$ as long as $\frac{\log n}{np_{\mathrm{obs}}}$ is sufficiently small. As a result, by picking
\begin{equation}
\epsilon=\frac{\frac{w_{\max}^{6}}{100w_{\min}^{6}}\frac{\log n}{L}}{\frac{12}{5}\frac{np_{\mathrm{obs}}}{w_{\min}}}=\frac{w_{\max}^{6}}{240w_{\min}^{5}}\frac{\log n}{np_{\mathrm{obs}}L},\label{eq:choice-epsilon}
\end{equation}
one can make sure that for any $\left|\tau-\tau_{0}\right|\leq\epsilon$,
\begin{align}
\frac{1}{L}\log\mathcal{L}\left(\tau,\hat{\boldsymbol{w}}_{\backslash i};\boldsymbol{y}_{i}\right) & -\frac{1}{L}\log\mathcal{L}\left(\tau_{0},\hat{\boldsymbol{w}}_{\backslash i};\boldsymbol{y}_{i}\right)\leq\epsilon\cdot\frac{12}{5}\frac{np_{\mathrm{obs}}}{w_{\min}},
\end{align}
\begin{align}
\Rightarrow\quad\frac{1}{L}\log\mathcal{L}\left(\tau,\hat{\boldsymbol{w}}_{\backslash i};\boldsymbol{y}_{i}\right) & <\frac{1}{L}\log\mathcal{L}\left(\tau_{0},\hat{\boldsymbol{w}}_{\backslash i};\boldsymbol{y}_{i}\right)+\frac{w_{\max}^{6}}{100w_{\min}^{6}}\frac{\log n}{L}.\label{eq:gap}
\end{align}
In addition, with the above choice of $\epsilon$, the cardinality of the $\epsilon$-cover is bounded above
by
\[
\left|\mathcal{N}_{\epsilon}\right|\leq\left\lceil \frac{w_{\max}}{\epsilon}\right\rceil +1=\left\lceil \frac{240np_{\mathrm{obs}}L}{\log n}\cdot\frac{w_{\min}^{5}}{w_{\max}^{5}}\right\rceil +1\ll n^{2}L
\]
for any sufficiently large $n$.

Putting (\ref{eq:discretize}) and (\ref{eq:gap}) together suggests
that for \emph{all} $\tau\in[w_{\min},w_{\max}]$ sufficiently apart
from the ground truth $w_{i}$, namely,
\begin{equation}
\forall\tau\in[w_{\min},w_{\max}]:\quad\left|\tau-w_{i}\right|\geq\frac{\left(6+\frac{\log L}{\log n}\right)w_{\max}^{5}}{w_{\min}^{4}}\max\left\{ \frac{25}{4}\left(\delta+\frac{\xi\log n}{np_{\mathrm{obs}}}\right),\text{ }\text{ }20\sqrt{\frac{\log n}{np_{\mathrm{obs}}L}}\right\} ,\label{eq:tau-distance}
\end{equation}
one necessarily has
\begin{align}
 & \frac{1}{L}\log\mathcal{L}\left(w_{i},\hat{\boldsymbol{w}}_{\backslash i};\boldsymbol{y}_{i}\right)-\frac{1}{L}\log\mathcal{L}\left(\tau,\hat{\boldsymbol{w}}_{\backslash i};\boldsymbol{y}_{i}\right)\nonumber \\
 & =\left\{ \frac{1}{L}\log\mathcal{L}\left(w_{i},\hat{\boldsymbol{w}}_{\backslash i};\boldsymbol{y}_{i}\right)-\frac{1}{L}\log\mathcal{L}\left(\tau_{0},\hat{\boldsymbol{w}}_{\backslash i};\boldsymbol{y}_{i}\right)\right\} +\left\{ \frac{1}{L}\log\mathcal{L}\left(\tau_{0},\hat{\boldsymbol{w}}_{\backslash i};\boldsymbol{y}_{i}\right)-\frac{1}{L}\log\mathcal{L}\left(\tau,\hat{\boldsymbol{w}}_{\backslash i};\boldsymbol{y}_{i}\right)\right\} \nonumber \\
 & >0,
\end{align}
with probability at least $1-6\left|\mathcal{N}_{\epsilon}\right|n^{-6-\frac{\log L}{\log n}}-O\left(n^{-4}\right)
\geq 1 - 6n^{2}Ln^{-6-\frac{\log L}{\log n}}-O(n^{-4})=1-O(n^{-4})$.
Consequently, any $\tau\in\left[w_{\min},w_{\max}\right]$ that obeys
(\ref{eq:tau-distance}) cannot be the coordinate-wise MLE, which
in turn justifies the claim (\ref{eq:MLEgap}) of Theorem \ref{thm:refinement}
(we present Theorem \ref{thm:refinement} using slightly looser constants for notational simplicity).

\begin{proof}[\textbf{of Lemma \ref{lemma-PointwiseError}}]
We start by evaluating the true coordinate-wise likelihood gap 
\begin{equation}
\log\mathcal{L}\left(w_{i},\boldsymbol{w}_{\backslash i};\boldsymbol{y}_{i}\right)-\log\mathcal{L}\left(\tau,\boldsymbol{w}_{\backslash i};\boldsymbol{y}_{i}\right)\label{eq:true-gap}
\end{equation}
for any fixed $\tau\neq w_{i}$ independent of $\boldsymbol{y}_{i}$.
Here, $\boldsymbol{y}_{i}:=\left\{ y_{i,j}\mid i: (i,j)\in\mathcal{E}\right\} $
is assumed to be generated under the BTL model parameterized by $\boldsymbol{w}$,
which clearly obeys
\[
\mathbb{E}\left[y_{i,j}\right]=\frac{w_{i}}{w_{i}+w_{j}}\quad\text{and}\quad{\bf Var}\left[y_{i,j}\right]=\frac{1}{L}\frac{w_{i}w_{j}}{\left(w_{i}+w_{j}\right)^{2}}.
\]
In order to calculate the mean of (\ref{eq:true-gap}), we
rewrite the likelihood function as
\begin{eqnarray}
\frac{1}{L}\log\mathcal{L}\left(\tau,\boldsymbol{w}_{\backslash i};\boldsymbol{y}_{i}\right) & = & \sum_{j:\left(i,j\right)\in\mathcal{E}}\left\{ y_{i,j}\log\left(\frac{\tau}{\tau+w_{j}}\right)+\left(1-y_{i,j}\right)\log\left(\frac{w_{j}}{\tau+w_{j}}\right)\right\} \label{eq:log-likelihood-original}\\
 & = & \sum_{j:\left(i,j\right)\in\mathcal{E}}y_{i,j}\log\left(\frac{\tau}{w_{j}}\right)+\sum_{j:\left(i,j\right)\in\mathcal{E}}\log\left(\frac{w_{j}}{\tau+w_{j}}\right).\label{eq:log-likelihood}
\end{eqnarray}
Taking expectation w.r.t. $\boldsymbol{y}_{i}$ using the form (\ref{eq:log-likelihood-original})
reveals that
\begin{align}
 & \mathbb{E}\left[\left.\frac{1}{L}\log\mathcal{L}\left(w_{i},\boldsymbol{w}_{\backslash i};\boldsymbol{y}_{i}\right)-\frac{1}{L}\log\mathcal{L}\left(\tau,\boldsymbol{w}_{\backslash i};\boldsymbol{y}_{i}\right)\right|\mathcal{G}\right] \nonumber\\
& \quad =\sum_{j:\left(i,j\right)\in\mathcal{E}}\left\{ \frac{w_{i}}{w_{i}+w_{j}}\log\left(\frac{\frac{w_{i}}{w_{i}+w_{j}}}{\frac{\tau}{\tau+w_{j}}}\right)+\frac{w_{j}}{w_{i}+w_{j}}\log\left(\frac{\frac{w_{j}}{w_{i}+w_{j}}}{\frac{w_{j}}{\tau+w_{j}}}\right)\right\} \nonumber \\
 & \quad=\sum_{j:\left(i,j\right)\in\mathcal{E}}\mathsf{KL}\left(\frac{w_{i}}{w_{i}+w_{j}}\text{ }\Big\|\text{ }\frac{\tau}{\tau+w_{j}}\right), \label{eq:LB_KL_E}
\end{align}
where $\mathsf{KL}\left(p\|q\right)$ stands for the KL divergence of $\text{Bernoulli}\left(q\right)$ from $\text{Bernoulli}\left(p\right)$.
Using Pinsker's inequality (e.g. \citep[Theorem 2.33]{yeung2008information}),
that is, $\mathsf{KL}\left(p\|q\right)\geq2\left(p-q\right)^{2}$,
we arrive at the following lower bound
\begin{eqnarray}
 &  & \mathbb{E}\left[\left.\frac{1}{L}\log\mathcal{L}\left(w_{i},\boldsymbol{w}_{\backslash i};\boldsymbol{y}_{i}\right)-\frac{1}{L}\log\mathcal{L}\left(\tau,\boldsymbol{w}_{\backslash i};\boldsymbol{y}_{i}\right)\right|\mathcal{G}\right]\geq2\sum_{j:\left(i,j\right)\in\mathcal{E}}\left(\frac{w_{i}}{w_{i}+w_{j}}-\frac{\tau}{\tau+w_{j}}\right)^{2}\nonumber \\
 &  & \quad=2\left(w_{i}-\tau\right)^{2}\sum_{j:\left(i,j\right)\in\mathcal{E}}\frac{w_{j}^{2}}{\left(w_{i}+w_{j}\right)^{2}\left(\tau+w_{j}\right)^{2}}.
\label{eq:PinskerLB}
\end{eqnarray}
That being said, the true coordinate-wise likelihood of $w_{i}$ strictly
dominates that of $\tau$ in the mean sense. 

Nevertheless, when running Spectral MLE, we do not have access to the ground
truth scores $\boldsymbol{w}_{\backslash i}$; what we actually compute
is $\mathcal{L}(w_{i},\hat{\boldsymbol{w}}_{\backslash i};\boldsymbol{y}_{i})$
(resp. $\mathcal{L}(\tau,\hat{\boldsymbol{w}}_{\backslash i};\boldsymbol{y}_{i})$)
rather than $\mathcal{L}\left(\boldsymbol{w};\boldsymbol{y}_{i}\right)$
(resp. $\mathcal{L}(\tau,\boldsymbol{w}_{\backslash i};\boldsymbol{y}_{i})$).
Fortunately, such surrogate likelihoods are sufficiently close to
the true coordinate-wise likelihoods, which we will show in the rest
of the proof. For brevity, we shall denote respectively the heuristic
and true log-likelihood functions by
\begin{equation}
\begin{cases}
\hat{\ell}_{i}\left(\tau\right) & :=\frac{1}{L}\log\mathcal{L}\left(\tau,\hat{\boldsymbol{w}}_{\backslash i};\boldsymbol{y}_{i}\right)\\
\ell^{*}\left(\tau\right) & :=\frac{1}{L}\log\mathcal{L}\left(\tau,\boldsymbol{w}_{\backslash i};\boldsymbol{y}_{i}\right)
\end{cases}\label{eq:LogLikelihood}
\end{equation}
whenever it is clear from context. Note that $\hat{\boldsymbol{w}}_{\backslash i}$
could depend on $\boldsymbol{y}_{i}$.

As seen from (\ref{eq:log-likelihood}), for any candidate $\tau\in\left[w_{\min},w_{\max}\right]$,
we can quantify the difference between $\hat{\ell}_{i}\left(\tau\right)$
and $\ell^{*}\left(\tau\right)$ as
\begin{eqnarray}
\hat{\ell}_{i}\left(\tau\right)-\ell^{*}\left(\tau\right)  ~=~ \sum_{j:\left(i,j\right)\in\mathcal{E}}y_{i,j}\log\left(\frac{w_{j}}{\hat{w}_{j}}\right)+\sum_{j:\left(i,j\right)\in\mathcal{E}}\left\{ \log\left(\frac{\hat{w}_{j}}{\tau+\hat{w}_{j}}\right)-\log\left(\frac{w_{j}}{\tau+w_{j}}\right)\right\} .
\end{eqnarray}
As a consequence, the gap between the true loss $\ell^{*}\left(w_{i}\right)-\ell^{*}\left(\tau\right)$
and the surrogate loss $\hat{\ell}_{i}\left(w_{i}\right)-\hat{\ell}_{i}\left(\tau\right)$
is given by
\begin{align}
 & \hat{\ell}_{i}\left(w_{i}\right)-\hat{\ell}_{i}\left(\tau\right)-\left(\ell^{*}\left(w_{i}\right)-\ell^{*}\left(\tau\right)\right)
~=~ \hat{\ell}_{i}\left(w_{i}\right)-\ell^{*}\left(w_{i}\right)-\left(\hat{\ell}_{i}\left(\tau\right)-\ell^{*}\left(\tau\right)\right)\nonumber \\
 & =\sum_{j:\left(i,j\right)\in\mathcal{E}}\left\{ \log\left(\frac{\hat{w}_{j}}{w_{i}+\hat{w}_{j}}\right)-\log\left(\frac{w_{j}}{w_{i}+w_{j}}\right)-\left(\log\left(\frac{\hat{w}_{j}}{\tau+\hat{w}_{j}}\right)-\log\left(\frac{w_{j}}{\tau+w_{j}}\right)\right)\right\} \\
 & =\sum_{j:\left(i,j\right)\in\mathcal{E}}\left\{ \log\left(\frac{\tau+\hat{w}_{j}}{w_{i}+\hat{w}_{j}}\right)-\log\left(\frac{\tau+w_{j}}{w_{i}+w_{j}}\right)\right\} .\label{eq:Lw}
\end{align}
This gap thus relies on the function
\[
g\left(t\right):=\log\left(\frac{\tau+t}{w_{i}+t}\right)-\log\left(\frac{\tau+w_{j}}{w_{i}+w_{j}}\right),\quad t\in\left[w_{\min},w_{\max}\right],
\]
which apparently obeys the following two properties: (i) $g\left(w_{j}\right)=0$;
(ii)
\[
\left|\frac{\partial g\left(t\right)}{\partial t}\right|=\left|\frac{1}{\tau+t}-\frac{1}{w_{i}+t}\right|=\frac{\left|\tau-w_{i}\right|}{\left(w_{i}+t\right)\left(\tau+t\right)}\leq\frac{\left|\tau-w_{i}\right|}{4w_{\min}^{2}},\quad\forall t\in\left[w_{\min},w_{\max}\right].
\]
Taken together these two properties demonstrate that
\begin{align}
\left|g\left(t\right)\right| &~\leq~ |g(w_j)| + |t-w_j| \cdot \sup_{t\in[w_{\min}, w_{\max}]} \left|\frac{\partial g\left(t\right)}{\partial t}\right| \nonumber\\
&~\leq~ \frac{1}{4w_{\min}^{2}}\left|\tau-w_{i}\right|\left|t-w_{j}\right|,\quad\forall t\in\left[w_{\min},w_{\max}\right].
\end{align}
Substitution into (\ref{eq:Lw}) gives
\begin{eqnarray}
\left|\hat{\ell}_{i}\left(w_{i}\right)-\hat{\ell}_{i}\left(\tau\right)-\left(\ell^{*}\left(w_{i}\right)-\ell^{*}\left(\tau\right)\right)\right| & \leq & \frac{1}{4w_{\min}^{2}}\left|\tau-w_{i}\right|\sum_{j:\left(i,j\right)\in\mathcal{E}}\left|\hat{w}_{j}-w_{j}\right|\nonumber \\
 & \leq & \frac{1}{4w_{\min}^{2}}\left|\tau-w_{i}\right|\sum_{j:\left(i,j\right)\in\mathcal{E}}\left|\hat{w}_{j}^{\mathrm{ub}}-w_{j}\right|.\label{eq:l_perturb_UB}
\end{eqnarray}
Notably, this is a deterministic inequality which holds for all $\hat{\boldsymbol{w}}_{j}$
obeying $|\hat{w}_{j}-w_{j}|\leq|\hat{w}_{j}^{\mathrm{ub}}-w_{j}|$, 
$1\leq j\leq n$. A desired property of the upper bound (\ref{eq:l_perturb_UB})
is that it is independent of $\mathcal{G}$ and the data $\boldsymbol{y}_{i}$,
due to our assumption on $\hat{\boldsymbol{w}}^{\mathrm{ub}}$.

We now move on to develop an upper bound on (\ref{eq:l_perturb_UB}).
From our assumptions on the initial estimate, we have
\[
\left\Vert \hat{\boldsymbol{w}}-\boldsymbol{w}\right\Vert ^{2}\leq\|\hat{\boldsymbol{w}}^{\mathrm{ub}}-\boldsymbol{w}\|^{2}\leq\delta^{2}\left\Vert \boldsymbol{w}\right\Vert ^{2}\leq nw_{\max}^{2}\delta^{2}.
\]
Since $\mathcal{G}$ and $\hat{\boldsymbol{w}}^{\mathrm{ub}}$ are
statistically independent, this inequality immediately gives rise
to the following two consequences:
\begin{eqnarray}
\mathbb{E}\left[\sum\nolimits{}_{j:\left(i,j\right)\in\mathcal{E}}\left|\hat{w}_{j}^{\mathrm{ub}}-w_{j}\right|\right] & = & p_{\mathrm{obs}}\|\hat{\boldsymbol{w}}^{\mathrm{ub}}-\boldsymbol{w}\|_{1}\leq p_{\mathrm{obs}}\sqrt{n}\|\hat{\boldsymbol{w}}^{\mathrm{ub}}-\boldsymbol{w}\|\nonumber \\
 & \leq & np_{\mathrm{obs}}w_{\max}\delta\label{eq:UB_w_hat_L1}
\end{eqnarray}
and
\begin{equation}
\mathbb{E}\left[\sum\nolimits{}_{j:\left(i,j\right)\in\mathcal{E}}\left|\hat{w}_{j}^{\mathrm{ub}}-w_{j}\right|^{2}\right]=p_{\mathrm{obs}}\|\hat{\boldsymbol{w}}^{\mathrm{ub}}-\boldsymbol{w}\|_{2}^{2}\leq np_{\mathrm{obs}}w_{\max}^{2}\delta^{2}.\label{eq:UB_w_hat_var}
\end{equation}
Recall our assumption that $\max_{j}\left|\hat{w}_{j}^{\mathrm{ub}}-w_{j}\right|\leq\xi w_{\max}$.
For any fixed $\gamma\geq3$, if $p_{\mathrm{obs}}>\frac{2\log n}{n}$,
then with probability at least $1-2n^{-\gamma}$,
\begin{eqnarray*}
\sum_{j:\left(i,j\right)\in\mathcal{E}}\left|\hat{w}_{j}^{\mathrm{ub}}-w_{j}\right| 
 & \overset{(\text{i})}{\leq} & \mathbb{E}\left[\sum_{j:\left(i,j\right)\in\mathcal{E}}\left|\hat{w}_{j}^{\mathrm{ub}}-w_{j}\right|\right]+\sqrt{2\gamma\log n\cdot\mathbb{E}\left[\sum{}_{j:\left(i,j\right)\in\mathcal{E}}\left|\hat{w}_{j}^{\mathrm{ub}}-w_{j}\right|^{2}\right]} \nonumber\\
 & & \qquad +\frac{2\gamma}{3}\xi w_{\max}\log n\\
 & \leq & np_{\mathrm{obs}}w_{\max}\delta+\sqrt{2\gamma\cdot np_{\mathrm{obs}}\log n}w_{\max}\delta+\frac{2\gamma}{3}\xi w_{\max}\log n\\
 & \overset{(\text{ii})}{\leq} & \left(1+\sqrt{\gamma}\right)np_{\mathrm{obs}}w_{\max}\delta+\frac{2\gamma}{3}\xi w_{\max}\log n\\
 & \overset{(\text{iii})}{\leq} & \gamma np_{\mathrm{obs}}w_{\max}\delta+\gamma\xi w_{\max}\log n,
\end{eqnarray*}
where (i) comes from the Bernstein inequality as given in Lemma \ref{lemma:Bernstein},
(ii) follows since $\log n<\frac{p_{\mathrm{obs}}n}{2}$ by assumption,
and (iii) arises since $1+\sqrt{\gamma}\leq\gamma$ whenever $\gamma\ge3$.
This combined with (\ref{eq:l_perturb_UB}) allows us to control
\begin{equation}
\left|\hat{\ell}_{i}\left(w_{i}\right)-\hat{\ell}_{i}\left(\tau\right)-\left(\ell^{*}\left(w_{i}\right)-\ell^{*}\left(\tau\right)\right)\right|\leq\frac{\left|\tau-w_{i}\right|\gamma w_{\max}}{4w_{\min}^{2}}\left(np_{\mathrm{obs}}\delta+\xi\log n\right)\label{eq:eq:gap_l_lperturb}
\end{equation}
with high probability.

The above arguments basically reveal that $\hat{\ell}_{i}\left(w_{i}\right)-\hat{\ell}_{i}\left(\tau\right)$
is reasonably close to $\ell^{*}\left(w_{i}\right)-\ell^{*}\left(\tau\right)$.
Thus, to show that $\hat{\ell}_{i}\left(w_{i}\right)-\hat{\ell}_{i}\left(\tau\right)>0$,
it is sufficient to develop a lower bound on $\ell^{*}\left(w_{i}\right)-\ell^{*}\left(\tau\right)$
that exceeds the gap (\ref{eq:eq:gap_l_lperturb}). In expectation,
the preceding inequality (\ref{eq:PinskerLB}) gives
\begin{eqnarray}
\mathbb{E}\left[\left.\ell^{*}\left(w_{i}\right)-\ell^{*}\left(\tau\right)\text{ }\right|\mathcal{G}\right] & \geq & 2\left(w_{i}-\tau\right)^{2}\sum_{j:\left(i,j\right)\in\mathcal{E}}\frac{w_{j}^{2}}{\left(w_{i}+w_{j}\right)^{2}\left(\tau+w_{j}\right)^{2}}\nonumber \\
 & \geq & \frac{w_{\min}^{2}}{8w_{\max}^{4}}\left(w_{i}-\tau\right)^{2}\mathrm{deg}\left(i\right).\label{eq:mean-LB}
\end{eqnarray}
Recognizing that $y_{i,j}=\frac{1}{L}\sum_{l=1}^{L}y_{i,j}^{\left(l\right)}$
is a sum of independent random variables $y_{i,j}^{\left(l\right)}\sim\text{Bernoulli}\left(\frac{w_{i}}{w_{i}+w_{j}}\right)$,
we can control the conditional variance as
\begin{align}
 & {\bf Var}\left[\left.\ell^{*}\left(w_{i}\right)-\ell^{*}\left(\tau\right)\text{ }\right|\mathcal{G}\right]  ~\overset{(\text{a})}{=} ~{\bf Var}\left[\left.\sum_{j:\left(i,j\right)\in\mathcal{E}}y_{i,j}\log\left(\frac{w_{i}}{\tau}\right)\text{ }\right|\mathcal{G}\right]\nonumber \\
 & \quad=\text{ }\log^{2}\left(\frac{w_{i}}{\tau}\right)\sum_{j:\left(i,j\right)\in\mathcal{E}}\frac{1}{L}\frac{w_{i}w_{j}}{\left(w_{i}+w_{j}\right)^{2}} ~\overset{(\text{b})}{\leq}~ \frac{1}{L}\frac{\left(w_{i}-\tau\right)^{2}}{\min\left\{ w_{i}^{2},\tau^{2}\right\} }\sum_{j:\left(i,j\right)\in\mathcal{E}}\frac{w_{\max}^{2}}{4w_{\min}^{2}}\nonumber \\
 & \quad\leq\text{ }\frac{w_{\max}^{2}}{4w_{\min}^{4}}\cdot\frac{1}{L}\left(w_{i}-\tau\right)^{2}\text{deg}\left(i\right),\label{eq:var-UB}
\end{align}
where (a) is an immediate consequence of (\ref{eq:log-likelihood}),
and (b) follows since $\left|\log\frac{\beta}{\alpha}\right|\leq\frac{\beta-\alpha}{\alpha}$
for any $\beta>\alpha>0$. Note that $0\leq\frac{1}{L}y_{i,j}^{\left(l\right)}\leq\frac{1}{L}$, and hence each summand of $\ell^{*}\left(w_{i}\right)-\ell^{*}\left(\tau\right)$ (written in terms of a weighted sum of $y_{i,j}^{(l)}$) is bounded in magnitude by 
\begin{equation}
	\max_{i,j,l}\frac{1}{L}y_{i,j}^{\left(l\right)}\left|\log\frac{w_i}{\tau}\right| ~\leq~ \frac{1}{L}\left|\log\frac{w_i}{\tau}\right| ~\leq~ \frac{1}{L} \frac{|w_i-\tau|}{w_{\min}}, \label{eq:max-mag}
\end{equation}
where the last
inequality follows again from the inequality $\left|\log\left(\frac{\beta}{\alpha}\right)\right|\leq\frac{\beta-\alpha}{\alpha}$
for any $\beta\geq\alpha>0$.
Making use of the Bernstein inequality together with (\ref{eq:mean-LB})-(\ref{eq:max-mag})
suggests that: conditional on $\mathcal{G}$,
\begin{align}
&  \ell^{*}\left(w_{i}\right)-\ell^{*}\left(\tau\right) \nonumber\\
&  \quad\geq~  \mathbb{E}\left[\left.\ell^{*}\left(w_{i}\right)-\ell^{*}\left(\tau\right)\text{ }\right|\mathcal{G}\right]-\sqrt{2\gamma{\bf Var}\left[\left.\ell^{*}\left(w_{i}\right)-\ell^{*}\left(\tau\right)\text{ }\right|\mathcal{G}\right]\log n}-\frac{2\gamma\log n}{3}\cdot\frac{\left|\log\left(\frac{w_{i}}{\tau}\right)\right|}{L}\nonumber \\
&  \quad\geq~  \frac{w_{\min}^{2}}{8w_{\max}^{4}}\left(w_{i}-\tau\right)^{2}\mathrm{deg}\left(i\right)-\frac{\sqrt{2\gamma}w_{\max}\left|w_{i}-\tau\right|}{2w_{\min}^{2}}\sqrt{\frac{\text{deg}\left(i\right)\log n}{L}}-\frac{2\gamma\left|w_{i}-\tau\right|\log n}{3Lw_{\min}}\label{eq:LB_trueLL}
\end{align}
holds with probability at least $1-2n^{-\gamma}$.
The above bound relies on $\mathrm{deg}(i)$, which is on the order
of $np_{\mathrm{obs}}$ with high probability. More precisely, taking
the Chernoff bound \citep[Corollary 4.6]{mitzenmacher2005probability}
as well as the union bound reveals that: if $\frac{\log n}{np_{\mathrm{obs}}}$ is sufficiently large,
then
\begin{equation}
\frac{4}{5}np_{\mathrm{obs}}<\mathrm{deg}\left(i\right)<\frac{6}{5}np_{\mathrm{obs}},\qquad 1\leq i\leq n
\end{equation}
with probability at least $1- 2n^{-\gamma}$. This taken collectively
with (\ref{eq:LB_trueLL}) and the assumption $np_{\mathrm{obs}}>2\log n$
implies that
\begin{align}
& \ell^{*}\left(w_{i}\right)-\ell^{*}\left(\tau\right)  \nonumber\\
& \quad\geq  \frac{w_{\min}^{2}}{8w_{\max}^{4}}\left(w_{i}-\tau\right)^{2}\cdot\frac{4}{5}np_{\mathrm{obs}}-\sqrt{\frac{\gamma}{2}}\frac{w_{\max}\left|w_{i}-\tau\right|}{w_{\min}^{2}}\sqrt{\frac{6np_{\mathrm{obs}}\log n}{5L}}-\frac{2\gamma\left|w_{i}-\tau\right|\log n}{3Lw_{\min}}\nonumber \\
 & \quad\geq  \frac{w_{\min}^{2}}{10w_{\max}^{4}}\left(w_{i}-\tau\right)^{2}np_{\mathrm{obs}}-\left(\sqrt{\frac{3\gamma}{5}}+\frac{2\gamma}{3}\frac{1}{\sqrt{2}}\right)\frac{w_{\max}\left|w_{i}-\tau\right|}{w_{\min}^{2}}\sqrt{\frac{np_{\mathrm{obs}}\log n}{L}}\nonumber \\
 & \quad\geq  \frac{w_{\min}^{2}}{10w_{\max}^{4}}\left(w_{i}-\tau\right)^{2}np_{\mathrm{obs}}-\gamma\frac{w_{\max}\left|w_{i}-\tau\right|}{w_{\min}^{2}}\sqrt{\frac{np_{\mathrm{obs}}\log n}{L}}\\
 & \quad\geq  \frac{w_{\min}^{2}}{20w_{\max}^{4}}\left(w_{i}-\tau\right)^{2}np_{\mathrm{obs}}
\label{eq:lstar_LB_balance}
\end{align}
with probability at least $1-4n^{-\gamma}$, as long as
\[
\gamma\cdot\frac{w_{\max}\left|w_{i}-\tau\right|}{w_{\min}^{2}}\sqrt{\frac{np_{\mathrm{obs}}\log n}{L}}\leq\frac{w_{\min}^{2}}{20w_{\max}^{4}}\left(w_{i}-\tau\right)^{2}np_{\mathrm{obs}}
\]
or, equivalently,
\begin{equation}
\left|w_{i}-\tau\right|\geq\frac{20\gamma\cdot w_{\max}^{5}}{w_{\min}^{4}}\sqrt{\frac{\log n}{np_{\mathrm{obs}}L}}.\label{eq:Condition3}
\end{equation}

Finally, we are ready to control $\hat{\ell}_{i}\left(w_{i}\right)-\hat{\ell}_{i}\left(\tau\right)$
from below. Putting (\ref{eq:eq:gap_l_lperturb}) and (\ref{eq:lstar_LB_balance})
together, we see that with high probability,
\begin{eqnarray}
\hat{\ell}_{i}\left(w_{i}\right)-\hat{\ell}_{i}\left(\tau\right) & \geq & \ell^{*}\left(w_{i}\right)-\ell^{*}\left(\tau\right)-\frac{\left|\tau-w_{i}\right|\gamma w_{\max}\left(np_{\mathrm{obs}}\delta+\xi\log n\right)}{4w_{\min}^{2}}\nonumber \\
 & \geq & \frac{w_{\min}^{2}}{20w_{\max}^{4}}\left(w_{i}-\tau\right)^{2}np_{\mathrm{obs}}-\frac{\left|\tau-w_{i}\right|\gamma w_{\max}}{4w_{\min}^{2}}\left(np_{\mathrm{obs}}\delta+\xi\log n\right)\nonumber \\
 & > & \frac{w_{\min}^{2}}{100w_{\max}^{4}}\left(w_{i}-\tau\right)^{2}np_{\mathrm{obs}}\label{eq:3rd-inequality}\\
 & > & \frac{w_{\max}^{6}}{100w_{\min}^{6}}\frac{\log n}{L},\label{eq:4th-inequality}
\end{eqnarray}
where (\ref{eq:3rd-inequality}) holds under the condition
\begin{equation}
\left|\tau-w_{i}\right|>\frac{25\gamma w_{\max}^{5}}{4w_{\min}^{4}}\left(\delta+\frac{\xi\log n}{np_{\mathrm{obs}}}\right),
\label{eq:sep2}
\end{equation}
and (\ref{eq:4th-inequality}) follows from the assumption (\ref{eq:Condition3}).
The claim (\ref{eq:Likelihood}) is then established under the conditions (\ref{eq:Condition3}) and (\ref{eq:sep2}).\end{proof}

\subsection{Proof of Theorem \ref{thm:SpectralMLE}}

The accuracy of top-$K$ identification is closely related to the
$\ell_{\infty}$ error of the score estimate. In the sequel, we shall
assume that $w_{\max}=1$ to simplify presentation. Our goal is
to demonstrate that
\begin{equation}
\left\Vert \boldsymbol{w}^{(t)}-\boldsymbol{w}\right\Vert _{\infty} ~\lesssim~ \sqrt{\frac{\log n}{np_{\mathrm{obs}}L}}+\frac{1}{2^{t}}\sqrt{\frac{\log n}{p_{\mathrm{obs}}L}} ~\asymp~ \xi_{t},\quad\forall t\in\mathbb{N},\label{eq:Linf-error}
\end{equation}
where
\begin{equation}
\xi_{t}:=c_{3}\left\{ \xi_{\min}+\frac{1}{2^{t}}\left(\xi_{\max}-\xi_{\min}\right)\right\} ,\quad\forall t\geq-1\label{eq:defn-Delta}
\end{equation}
with $\xi_{\min}=\sqrt{\frac{\log n}{np_{\mathrm{obs}}L}}$ and $\xi_{\max}=\sqrt{\frac{\log n}{p_{\mathrm{obs}}L}}$.
If $T\geq c_{2}\log n$ for some sufficiently large $c_{2}>0$, then
this gives
\begin{align*}
\left\Vert \boldsymbol{w}^{(T)}-\boldsymbol{w}\right\Vert _{\infty} 
~\lesssim~ \sqrt{\frac{\log n}{np_{\mathrm{obs}}L}}=\xi_{\min}.
\end{align*}
The key implication is the following: if $w_{K}-w_{K-1}\geq c_{1}\sqrt{\frac{\log n}{np_{\mathrm{obs}}L}}$
for some sufficiently large $c_{1}>0$, then 
\begin{align*}
&w_{i}^{(T)}-w_{j}^{(T)}  ~\geq~ w_{i}-w_{j}-\left|w_{i}^{(T)}-w_{i}\right|-\left|w_{j}^{(T)}-w_{j}\right| \\
&\quad\geq~ w_{K}-w_{K+1}-2\left\Vert \boldsymbol{w}^{(T)}-\boldsymbol{w}\right\Vert_{\infty} >0
\end{align*}
for all $1\leq i\leq K$ and $j\geq K+1$, indicating that Spectral
MLE will output the first $K$ items as desired. The remaining proof
then comes down to showing (\ref{eq:Linf-error}). 

We start from $t=0$. When the initial estimate $\boldsymbol{w}^{(0)}$
is computed by Rank Centrality, the $\ell_{2}$ estimation error satisfies
\citep{Negahban2012}
\begin{equation}
\frac{\left\Vert \boldsymbol{w}^{(0)}-\boldsymbol{w}\right\Vert }{\left\Vert \boldsymbol{w}\right\Vert }\leq c_{4}\sqrt{\frac{\log n}{np_{\mathrm{obs}}L}}=c_{4}\xi_{\min}:=\delta\label{eq:MSE-w0}
\end{equation}
with high probability, where $c_{4}>0$ is some universal constant
independent of $n,p_{\mathrm{obs}},L$ and $\Delta_{K}$. A by-product
of this result is an upper bound
\begin{equation}
\left\Vert \boldsymbol{w}^{(0)}-\boldsymbol{w}\right\Vert _{\infty}\leq\left\Vert \boldsymbol{w}^{(0)}-\boldsymbol{w}\right\Vert \leq\delta\|\boldsymbol{w}\|\leq\delta\sqrt{n}=c_{4}\sqrt{\frac{\log n}{p_{\mathrm{obs}}L}},\label{eq:Linf_w0}
\end{equation}
which together with the fact $\left\Vert \boldsymbol{w}^{(0)}-\boldsymbol{w}\right\Vert _{\infty}\leq w_{\max}-w_{\min}\leq1$
gives
\begin{equation}
\left\Vert \boldsymbol{w}^{(0)}-\boldsymbol{w}\right\Vert _{\infty}\leq\min\left\{ c_{4}\sqrt{\frac{\log n}{p_{\mathrm{obs}}L}},\text{ 1}\right\} =\min\left\{ c_{4}\xi_{\max},1\right\} .\label{eq:Linf_w0-1}
\end{equation}
This justifies that $\boldsymbol{w}^{(0)}$ satisfies the claim (\ref{eq:Linf-error}).
Notably, $\boldsymbol{w}^{(0)}$ is independent of $\mathcal{E}^{\mathrm{iter}}$
and $\boldsymbol{y}^{\mathrm{iter}}$ and, therefore, independent
of the iterative steps. 

In what follows, we divide the iterative stage into two phases: (1)
$t\leq T_{0}$ and (2) $t>T_{0}$, where $T_{0}$ is a threshold such
that
\begin{equation}
\xi_{t}~\geq~ c_{10}\xi_{\min}~=~c_{10}\sqrt{\frac{\log n}{np_{\mathrm{obs}}L}},\qquad\text{iff}\quad t\leq T_{0},\label{eq:T0-defn}
\end{equation}
for some large constant $c_{10}>0$. As is seen from the definition
of $\xi_{t}$, $T_{0}\lesssim\log n$ holds as long as $L=O\left(\mathrm{poly}\left(n\right)\right)$. 

For the case where $t\leq T_{0}$, we proceed by induction on $t$
w.r.t. the following hypotheses:
\begin{itemize}
\item $\mathcal{M}_{t}$: $\left\Vert \boldsymbol{w}^{\mathrm{mle}}-\boldsymbol{w}\right\Vert _{\infty} < \frac{1}{2}\xi_{t}$
holds at the $t^{\mathrm{th}}$ iteration (the iteration where we
compute $\boldsymbol{w}^{(t+1)}$);
\item $\mathcal{B}_{t}$: \emph{all} entries $w_{i}^{(\tau)}$ of $\boldsymbol{w}^{(\tau)}$
($\tau\leq t-1$) satisfying $|w_{i}^{(\tau)}-w_{i}|\geq1.5\xi_{t}$
have been replaced by time $t$;
\item $\mathcal{H}_{t}$: \emph{none} of the entries $w_{i}^{(\tau)}$ ($\tau\leq t-1$)
satisfying $|w_{i}^{(\tau)}-w_{i}|\leq\frac{1}{2}\xi_{t}$ have been
replaced by time $t$.
\end{itemize}
We first note that $\mathcal{B}_{t}$ is an immediate
consequence of $\mathcal{M}_{t}$ and $\mathcal{B}_{t-1}$. In fact, given $\mathcal{B}_{t-1}$,  it suffices to examine
those entries $w_{i}^{(\tau)}$ 
that have \emph{not} been replaced by time $t-1$. To this end, we recall
that Spectral MLE replaces $w_{i}^{(\tau)}$ if $|w_{i}^{(\tau)}-w_{i}^{\mathrm{mle}}|>\xi_{t}$.
With $\mathcal{M}_{t}$ in place, for each $i$ obeying $|w_{i}^{(\tau)}-w_{i}|\geq1.5\xi_{t}$,
one has
\[
|w_{i}^{(\tau)}-w_{i}^{\mathrm{mle}}|\geq|w_{i}^{(\tau)}-w_{i}|-|w_{i}^{\mathrm{mle}}-w_{i}|>1.5\xi_{t}-\frac{1}{2}\xi_{t}=\xi_{t}
\]
and hence it will necessarily be replaced by $w_{i}^{\mathrm{mle}}$ at
time $t$. Similarly, $\mathcal{H}_{t}$ is an immediate
consequence of $\mathcal{M}_{t}$ and $\mathcal{H}_{t-1}$.\footnote{ Given $\mathcal{M}_{t}$ and $\mathcal{H}_{t-1}$,
for any $i$ obeying $|w_{i}^{(\tau)}-w_{i}|\leq0.5\xi_{t}$,
one has
\[
|w_{i}^{(\tau)}-w_{i}^{\mathrm{mle}}|\leq|w_{i}^{(\tau)}-w_{i}|+|w_{i}^{\mathrm{mle}}-w_{i}|<\frac{1}{2}\xi_{t}+\frac{1}{2}\xi_{t}=\xi_{t}
\]
and, therefore, it cannot be replaced by time $t$, which establishes $\mathcal{H}_t$.} As a consequence, it boils down
to verifying $\mathcal{M}_{t}$. 

When $t=0$, applying Theorem \ref{thm:refinement} and setting $\hat{\boldsymbol{w}}^{\mathrm{ub}}=\boldsymbol{w}^{(0)}$,
we see that
\[
\left\Vert \boldsymbol{w}^{\mathrm{mle}}-\boldsymbol{w}\right\Vert _{\infty}\leq c_{7}\xi_{\min}+c_{9}\frac{\log n}{np_{\mathrm{obs}}}\xi_{\max}
\]
for some universal constants $c_{7},c_{9}>0$, where we have made
use of the properties (\ref{eq:MSE-w0}) and (\ref{eq:Linf_w0-1}).
When $c_{10}$ is sufficiently large, the definition of $T_{0}$ (cf.
(\ref{eq:T0-defn})) gives $\xi_{0}\gg c_{7}\sqrt{\frac{\log n}{np_{\mathrm{obs}}L}}$;
additionally, $c_{9}\frac{\log n}{np_{\mathrm{obs}}}\xi_{\max}\ll\xi_{\max}\leq\xi_{0}$
holds as long as $\frac{\log n}{np_{\mathrm{obs}}}$ is sufficiently
small. Putting these conditions together gives
\[
\left\Vert \boldsymbol{w}^{\mathrm{mle}}-\boldsymbol{w}\right\Vert _{\infty}\leq c_{7}\xi_{\min}+c_{9}c_{4}\frac{\log n}{np_{\mathrm{obs}}}\xi_{\max}<\frac{1}{2}\xi_{0},
\]
which verifies the property $\mathcal{M}_{0}$. 

We now turn to extending these inductive hypotheses to the $t^{\mathrm{th}}$
iteration, assuming that all of them hold up to time $t-1$. Taken
together $\mathcal{M}_{t-1}$ and $\mathcal{B}_{t-1}$ immediately
reveal that 
\begin{equation}
\left\Vert \boldsymbol{w}^{(t)}-\boldsymbol{w}\right\Vert _{\infty}\leq1.5\xi_{t-1}.\label{eq:loss-t}
\end{equation}
In order to invoke Theorem \ref{thm:refinement} for the coordinate-wise
MLEs, we need to construct a looser auxiliary score estimate $\hat{\boldsymbol{w}}^{\mathrm{ub}}$.
With $\mathcal{B}_{t-1}$, $\mathcal{H}_{t-1}$ and (\ref{eq:loss-t})
in mind, we propose a candidate for the $t^{\mathrm{th}}$ iteration
as follows%
\footnote{Careful readers will note that when $|w_{i}^{(0)}-w_{i}|\geq\frac{1}{2}\xi_{t-1}$,
the resulting $\hat{w}_{i}^{\mathrm{ub}}$ might exceed the range $[w_{\min},w_{\max}]$.
This can be easily addressed if we do the following: (1) change $\hat{w}_{i}^{\mathrm{ub}}$
to $w_{i}-1.5\xi_{t-1}$ instead if $w_{i}-1.5\xi_{t-1}\in[w_{\min},w_{\max}]$;
(2) if it is still infeasible, set $\hat{w}_{i}^{\mathrm{ub}}$ to be $w_{\max}$
if $|w_{i}-w_{\max}|>|w_{i}-w_{\min}|$ and $w_{\min}$ otherwise.
For simplicity of presentation, however, we omit these boundary situations
and assume $w_{i}+1.5\xi_{t-1}\leq w_{\max}$ throughout, which
will not change the results anyway.%
}
\begin{equation}
\hat{w}_{i}^{\mathrm{ub}}=\begin{cases}
w_{i}+1.5\xi_{t-1},\quad & \text{if }|w_{i}^{(0)}-w_{i}|>\frac{1}{2}\xi_{t-1},\\
w_{i}^{(0)} & \text{else}.
\end{cases}\label{eq:construction-wub}
\end{equation}
which is clearly independent of $\mathcal{E}^{\mathrm{iter}}$ and
$\boldsymbol{y}^{\mathrm{iter}}$. According to $\mathcal{B}_{t-1}$
and $\mathcal{H}_{t-1}$, (i) none of the entries $w_{i}^{(0)}$ with
$|w_{i}^{(0)}-w_{i}|\leq\frac{1}{2}\xi_{t-1}$ have been replaced
so far; (ii) if an entry $w_{i}^{(0)}$ has ever been replaced, then
the error of the new iterate cannot exceed $1.5\xi_{t-1}$ (otherwise
it'll be replaced by the MLE in time $t-1$ which gives an error below
$0.5\xi_{t-1}$). As a result, $\hat{\boldsymbol{w}}^{\mathrm{ub}}$ 
satisfies
\begin{equation}
\left|w_{i}^{(t)}-w_{i}\right| ~\leq~ \left|\hat{w}_{i}^{\mathrm{ub}}-w_{i}\right| ~\le~ 1.5\xi_{t-1},
\end{equation}
\begin{align}
\text{and}\quad\left\Vert \boldsymbol{w}^{(t)}-\boldsymbol{w}\right\Vert  & \leq\left\Vert \boldsymbol{w}^{(\mathrm{ub})}-\boldsymbol{w}\right\Vert ~\overset{(\text{a})}{\leq} ~ 3\left\Vert \boldsymbol{w}^{(0)}-\boldsymbol{w}\right\Vert ~\leq~ 3\delta\|\boldsymbol{w}\|.
\end{align}
Here, (a) arises since: (1) due to $\mathcal{H}_{t-1}$, if $ $$w_{i}^{(0)}$ is ever replaced, then $\big|w_{i}^{(0)}-w_{i}\big|$ is at least $0.5\xi_{t-1}$; (2) by construction, the pointwise error of $\hat{\boldsymbol{w}}^{\mathrm{ub}}$ is at most $1.5\xi_{t-1}$, and hence the replacement cannot inflate
the original error $\big|w_{i}^{(0)}-w_{i}\big|$ by more than $\frac{1.5\xi_{t-1}}{0.5\xi_{t-1}}=3$ times. With these in place, applying
Theorem \ref{thm:refinement} gives
\[
\left\Vert \boldsymbol{w}^{\mathrm{mle}}-\boldsymbol{w}\right\Vert _{\infty}\leq c_{8}\xi_{\min}+1.5c_{9}\frac{\log n}{np_{\mathrm{obs}}}\xi_{t-1},
\]
which relies on the fact $\delta\lesssim\sqrt{\frac{\log n}{np_{\mathrm{obs}}L}}$.
Recognize that
\[
\xi_{t}\gg c_{8}\xi_{\min}\quad\text{and}\quad1.5c_{9}\frac{\log n}{np_{\mathrm{obs}}}\xi_{t-1}\ll\xi_{t}
\]
hold in the regime where $t\leq T_{0}$ and $\frac{\log n}{np_{\mathrm{obs}}}\ll1$,
which taken together give
\[
\left\Vert \boldsymbol{w}^{\mathrm{mle}}-\boldsymbol{w}\right\Vert _{\infty} < \frac{1}{2}\xi_{t}
\]
as claimed in $\mathcal{M}_{t}$. Having verified these inductive
hypotheses, we see from the above argument that in any event, the $\ell_{\infty}$
error bound at the $t^{\mathrm{th}}$ iteration is at most $1.5\xi_{t}$,
which in turn leads to the claim (\ref{eq:Linf-error}) for any $t\leq T_{0}$.

Starting from $t=T_{0}+1$, we fix the auxiliary score as follows
\begin{equation}
\hat{w}_{i}^{\mathrm{ub}}=\begin{cases}
w_{i}+1.5\xi_{T_{0}},\quad & \text{if }|w_{i}^{(0)}-w_{i}|>\frac{1}{2}\xi_{\infty},\\
w_{i}^{(0)} & \text{else},
\end{cases}\label{eq:construction-wub-1}
\end{equation}
where we recall that $\xi_{\infty}=c_{3}\xi_{\min}$ and $\xi_{T_{0}}=c_{10}\xi_{\min}$.
This satisfies
\[
\left|w_{i}^{(t)}-w_{i}\right|\leq\left|\hat{w}_{i}^{\mathrm{ub}}-w_{i}\right|\leq1.5\xi_{T_{0}}
\]
for $t=T_{0}+1$, due to $\mathcal{M}_{T_0}$ and $\mathcal{B}_{T_0}$. Moreover, the number of indices that satisfy $|w_{i}^{(0)}-w_{i}|>\frac{1}{2}\xi_{\infty}$,
denoted by $k$, obeys
\[
k\cdot\left(\frac{1}{2}\xi_{\infty}\right)^{2}\leq\left\Vert \boldsymbol{w}-\boldsymbol{w}^{(0)}\right\Vert ^{2}\leq\delta^{2}\|\boldsymbol{w}\|^{2}\quad\Longleftrightarrow\quad k \leq\frac{4\delta^{2}\|\boldsymbol{w}\|^{2}}{\xi_{\infty}^{2}},
\]
which further gives
\begin{align*}
\left\Vert \boldsymbol{w}^{\mathrm{ub}}-\boldsymbol{w}\right\Vert ^{2} & \leq\left\Vert \boldsymbol{w}^{(0)}-\boldsymbol{w}\right\Vert ^{2}+\sum_{i:\text{ }|w_{i}^{(0)}-w_{i}|>\frac{1}{2}\Delta_{\infty}}\left(1.5\xi_{T_{0}}\right)^{2}\leq\delta^{2}\|\boldsymbol{w}\|^{2}+2.25k\xi_{T_{0}}^{2}\\
 & \leq\delta^{2}\|\boldsymbol{w}\|^{2}\left(1+\frac{9\xi_{T_{0}}^{2}}{\xi_{\infty}^{2}}\right).
\end{align*}
Note that the preceding analysis does not depend on the ratio $\frac{c_{10}}{c_{3}}$ as long as both $c_3$ and $c_{10}$ are large. 
If we pick $\frac{c_{10}}{c_{3}} =\frac{\xi_{T_{0}}}{\xi_{\infty}}\leq\sqrt{2}$,
then the above inequality gives rise to
\[
\left\Vert \boldsymbol{w}^{\mathrm{ub}}-\boldsymbol{w}\right\Vert \leq\sqrt{19}\delta\|\boldsymbol{w}\|.
\]
Applying Theorem \ref{thm:refinement} we deduce
\[
\left\Vert \boldsymbol{w}^{\mathrm{mle}}-\boldsymbol{w}\right\Vert _{\infty}\lesssim\delta+\frac{\log n}{np_{\mathrm{obs}}}\xi_{T_{0}}+\sqrt{\frac{\log n}{np_{\mathrm{obs}}L}}\asymp\sqrt{\frac{\log n}{np_{\mathrm{obs}}L}}< \frac{1}{2}\xi_{\infty},
\]
as long as $\frac{\log n}{p_{\mathrm{obs}}n}$ is small and $c_{10},c_{3}$
are sufficiently large. 

The main point of the above calculation is that: for any entry $w_{i}^{(0)}$
satisfying $|w_{i}^{(0)}-w_{i}|<\frac{1}{2}\xi_{\infty}$, one must
have
\[
\left|w_{i}^{(0)}-w_{i}^{\mathrm{mle}}\right|\leq\left|w_{i}^{(0)}-w_{i}\right|+\left|w_{i}^{(\mathrm{mle})}-w_{i}\right|<\xi_{\infty}<\xi_{t},
\]
and hence it will not be replaced. As a result, the auxiliary score
(\ref{eq:construction-wub-1}) remains valid for the iteration that
follows. In fact, these properties continue to hold for all $t>T_0$ if we repeat the same argument as $t$ increases. 
To finish up, put together the above arguments  to obtain
\[
\left\Vert \boldsymbol{w}^{(t)}-\boldsymbol{w}\right\Vert _{\infty}\leq\frac{1}{2}\xi_{\infty}\asymp\sqrt{\frac{\log n}{np_{\mathrm{obs}}L}},\quad t>T_{0},
\]
which establishes the claim (\ref{eq:Linf-error}) for $t>T_{0}$
and, in turn, Theorem \ref{thm:SpectralMLE}.

\section{Proof of the Minimax Lower Bound (Theorem~\ref{thm:optimality}) \label{sec:converse-proof}}

This section establishes the minimax lower limit given in Theorem~\ref{thm:optimality}.
To bound the minimax probability of error, we proceed by constructing
a finite set of hypotheses, followed by an analysis based on classical
Fano-type arguments. For notational simplicity, each hypothesis is
represented by a permutation $\sigma$ over $[n]$, and we denote
by $\sigma(i)$ and $\sigma\left(\left[K\right]\right)$ the corresponding index
of the $i^{\text{th}}$ ranked item and the index set of all top-$K$
items, respectively.

We now single out a set of hypotheses and some prior to be imposed
on them. Suppose that the values of $\boldsymbol{w}$ are fixed up
to permutation in such a way that 
\[
w_{\sigma(i)}=\begin{cases}
w_{K},\quad & 1\leq i\leq K,\\
w_{K+1},\quad & K<i\leq n,
\end{cases}
\]
where we abuse the notation $w_{K},w_{K+1}$ to represent any two
values satisfying 
\[
\frac{w_{K}-w_{K+1}}{w_{\max}}=\Delta_{K}>0.
\]
Below we suppose that the ranking scheme is informed of the values
$w_{K},w_{K+1}$, which only makes the ranking task easier. In addition,
we impose a uniform prior over a collection $\mathcal{M}$ of $M:=\max\left\{ K,n-K\right\} +1$
hypotheses regarding the permutation: if $K<n/2$, then 
\begin{align}
\mathbb{P}\left\{ \sigma\left(\left[K\right]\right)=\mathcal{S}\right\} =\frac{1}{M},\qquad\text{for }\mathcal{S}=\left\{ 2,\cdots,K\right\} \cup\left\{ i\right\} ,\quad(i=1,K+1,\cdots,n);\label{eq:uniformprior}
\end{align}
if $K\geq n/2$, then 
\begin{align}
\mathbb{P}\left\{ \sigma\left(\left[K\right]\right)=\mathcal{S}\right\} =\frac{1}{M},\qquad\text{for }\mathcal{S}=\left\{ 1,\cdots,K+1\right\} \backslash\left\{ i\right\} ,\quad(i=1,\cdots,K+1).\label{eq:uniformprior2}
\end{align}
In words, each alternative hypothesis is generated by swapping two
indices of the hypothesis obeying $\sigma\left(\left[K\right]\right)=\left[K\right]$.
Denoting by $P_{\mathrm{e},M}$ the average probability of error with
respect to the prior we construct, one can easily verify that the
minimax probability of error is at least $P_{\mathrm{e},M}$.

This Bayesian probability of error will be bounded using classical
Fano-type bounds. To accommodate partial observation, we introduce
an erased version of $\boldsymbol{y}_{i,j}:=(y_{i,j}^{(1)},\cdots,y_{i,j}^{(L)})$
such that 
\[
\boldsymbol{z}_{i,j}=\begin{cases}
\boldsymbol{y}_{i,j},\quad & \text{with probability }p_{\mathrm{obs}},\\
\text{erasure},\quad & \text{else},
\end{cases}
\]
and set $\boldsymbol{Z}:=\left\{ \boldsymbol{z}_{i,j}\right\} _{1\leq i< j\leq n}$.
Applying the generalized Fano inequality \citep[Theorem 7]{han1994generalizing}
gives
\begin{align*}
P_{\mathrm{e},M} & \geq1-\frac{1}{\log M}\left\{ \frac{1}{M^{2}}\sum_{\sigma_{1},\sigma_{2}\in\mathcal{M}}\mathsf{KL}\left(\mathbb{P}_{\boldsymbol{Z}\mid\sigma=\sigma_{1}}\hspace{0.3em}\|\hspace{0.3em}\mathbb{P}_{\boldsymbol{Z}\mid\sigma=\sigma_{2}}\right)+\log2\right\} \\
 & \overset{(\text{a})}{=}1-\frac{1}{\log M}\left\{ \text{ }\frac{1}{M^{2}}\sum_{\sigma_{1},\sigma_{2}\in\mathcal{M}}\sum_{i\neq j}\mathsf{KL}\left(\mathbb{P}_{\boldsymbol{z}_{i,j}\mid\sigma=\sigma_{1}}\hspace{0.3em}\|\hspace{0.3em}\mathbb{P}_{\boldsymbol{z}_{i,j}\mid\sigma=\sigma_{2}}\right)+\log2\right\} \\
 & \overset{(\text{b})}{=}\text{ }1-\frac{1}{\log M}\left\{ \frac{p_{{\rm obs}}}{M^{2}}\sum_{\sigma_{1},\sigma_{2}\in\mathcal{M}}\sum_{i\neq j}\mathsf{KL}\left(\mathbb{P}_{\boldsymbol{y}_{i,j}\mid\sigma=\sigma_{1}}\hspace{0.3em}\|\hspace{0.3em}\mathbb{P}_{\boldsymbol{y}_{i,j}\mid\sigma=\sigma_{2}}\right)+\log2\right\} \\
 & \overset{(\text{c})}{=}1-\frac{1}{\log M}\left\{ \text{ }\frac{p_{{\rm obs}}L}{M^{2}}\sum_{\sigma_{1},\sigma_{2}\in\mathcal{M}}\sum_{i\neq j}\mathsf{KL}\left(\mathbb{P}_{y_{i,j}^{(1)}\mid\sigma=\sigma_{1}}\hspace{0.3em}\|\hspace{0.3em}\mathbb{P}_{y_{i,j}^{(1)}\mid\sigma=\sigma_{2}}\right)+\log2\right\} \\
 & \overset{(\text{d})}{\geq}1-\frac{1}{\log M}\left\{ \text{ }\frac{2w_{\max}^{4}}{w_{\min}^{4}}np_{{\rm obs}}L\Delta_{K}^{2}+\log2\right\} ,
\end{align*}
where $\mathsf{KL}\left(P\hspace{0.3em}\|\hspace{0.3em}Q\right)$
denotes the KL divergence of $Q$ from $P$. Here, (a)
comes from the independence assumption of the $\boldsymbol{z}_{i,j}$'s;
(b) arises since $\boldsymbol{z}_{i,j}$ is an erased version of $\boldsymbol{y}_{i,j}$;
(c) follows since $y_{i,j}^{(\ell)}$ ($1\leq l\leq L$) are i.i.d.;
and (d) arises from Lemma \ref{lemma:KLDcomputation} (see below).

Consequently, one would have $P_{\mathrm{e}}\geq P_{\mathrm{e},M}\geq\epsilon$
if 
\[
\frac{2w_{\max}^{4}}{w_{\min}^{4}}np_{\mathrm{obs}}L\Delta_{K}^{2}\leq\left(1-\epsilon\right)\log M-\log2.
\]
Since $\left|\mathcal{M}\right|=M\geq\frac{n}{2}$, the above condition
is necessarily satisfied when 
\[
\frac{2w_{\max}^{4}}{w_{\min}^{4}}np_{\mathrm{obs}}L\Delta_{K}^{2}\leq\left(1-\epsilon\right)\log n-2\quad\Longleftrightarrow\quad L\leq\frac{w_{\min}^{4}}{2w_{\max}^{4}}\cdot\frac{\left(1-\epsilon\right)\log n-2}{np_{\mathrm{obs}}\Delta_{K}^{2}},
\]
which finishes the proof.

\begin{lemma} \label{lemma:KLDcomputation}If $w_{K},w_{K+1}\in\left[w_{\min},w_{\max}\right]$,
then for any $\sigma_{1},\sigma_{2}\in\mathcal{M}$: 
\begin{align}
\sum_{i\neq j}\mathsf{KL}\left(\mathbb{P}_{y_{i,j}^{(1)}\mid\sigma=\sigma_{1}}\hspace{0.3em} \big\| \hspace{0.3em}\mathbb{P}_{y_{i,j}^{(1)}\mid\sigma=\sigma_{2}}\right)\leq\frac{2w_{\max}^{4}}{w_{\min}^{4}}n\Delta_{K}^{2}.\label{eq:lemma2}
\end{align}
\end{lemma} \begin{proof} To start with, for any two measures $P\sim\text{Bernoulli}\left(p\right)$
and $Q\sim\text{Bernoulli}\left(q\right)$, one has \citep[Eqn. (7)]{Erven2014}
\begin{equation}
\mathsf{KL}\left(P\hspace{0.3em}\|\hspace{0.3em}Q\right)\leq\chi^{2}\left(P\hspace{0.3em}\|\hspace{0.3em}Q\right)=\frac{\left(p-q\right)^{2}}{q}+\frac{\left(p-q\right)^{2}}{1-q}=\frac{\left(p-q\right)^{2}}{q\left(1-q\right)}.\label{eq:KL-UB}
\end{equation}
where $\chi^{2}\left(P\hspace{0.3em}\|\hspace{0.3em}Q\right)$ denotes
the $\chi^{2}$ divergence between $P$ and $Q$.

Recall that given $\sigma=\sigma_{1}$ (resp. $\sigma=\sigma_{2}$),
$y_{i,j}^{(1)}$ is Bernoulli distributed with mean $r_{1}:=\frac{w_{\sigma_{1}\left(i\right)}}{w_{\sigma_{1}\left(i\right)}+w_{\sigma_{1}\left(j\right)}}$
(resp. $r_{2}:=\frac{w_{\sigma_{2}\left(i\right)}}{w_{\sigma_{2}\left(i\right)}+w_{\sigma_{2}\left(j\right)}}$).
If we set $\delta=r_{1}-r_{2}$, then \eqref{eq:KL-UB} yields 
\begin{align*}
\mathsf{KL} & \left(\mathbb{P}_{y_{i,j}^{(1)}\mid\sigma=\sigma_{1}}\hspace{0.3em} \big\|\hspace{0.3em}\mathbb{P}_{y_{i,j}^{(1)}\mid\sigma=\sigma_{2}}\right)\leq\frac{\delta^{2}}{r_{2}\left(1-r_{2}\right)}\leq\frac{4w_{\max}^{2}}{w_{\min}^{2}}\delta^{2},
\end{align*}
where the last inequality follows since 
\[
r_{2}\left(1-r_{2}\right)=\frac{w_{\sigma_{2}\left(i\right)}w_{\sigma_{2}\left(j\right)}}{\left(w_{\sigma_{2}\left(i\right)}+w_{\sigma_{2}\left(j\right)}\right)^{2}}\geq\frac{w_{\min}^{2}}{4w_{\max}^{2}}.
\]

By construction, conditional on any hypotheses $\sigma_{1},\sigma_{2}\in\mathcal{M}$,
the distributions of $\boldsymbol{y}_{i,j}$ are different over at most $2n$
locations. For each of these $O\left(n\right)$ locations, our construction
of $\mathcal{M}$ ensures that 
\[
|\delta|=|r_{2}-r_{1}|\leq\frac{w_{K}}{w_{K}+w_{K+1}}-\frac{w_{K+1}}{w_{K}+w_{K+1}}=\frac{w_{K}-w_{K+1}}{w_{K}+w_{K+1}}\leq\frac{w_{\max}}{2w_{\min}}\Delta_{K}.
\]
As a result, the total contribution is bounded above by 
\begin{align*}
\sum_{i\neq j}\mathsf{KL} & \left(\mathbb{P}_{y_{i,j}^{(1)}\mid\sigma=\sigma_{1}}\hspace{0.3em} \big\|\hspace{0.3em}\mathbb{P}_{y_{i,j}^{(1)}\mid\sigma=\sigma_{2}}\right)\leq2n\cdot\left(\max_{i,j}\delta^{2}\right)\frac{4w_{\max}^{2}}{w_{\min}^{2}}\leq\frac{2w_{\max}^{4}}{w_{\min}^{4}}n\Delta_{K}^{2}.
\end{align*}

\end{proof}

\section{Bernstein Inequality\label{sec:Bernstein-Inequality}}

Our analysis relies on the Bernstein inequality. To simplify presentation,
we state below a user-friendly version of Bernstein inequality.

\begin{lemma}\label{lemma:Bernstein}Consider $n$ independent random
variables $z_{l}$ ($1\leq l\leq n$), each satisfying $\left|z_{l}\right|\leq B$.
For any
$a\geq2$, one has
\begin{equation}
\left|\sum_{l=1}^{n}z_{l}-\mathbb{E}\left[\sum_{l=1}^{n}z_{l}\right]\right|\leq\sqrt{2a\log n\sum_{l=1}^{n}\mathbb{E}\left[z_{l}^{2}\right]}+\frac{2a}{3}B\log n\label{eq:Bernstein}
\end{equation}
with probability at least $1-\frac{2}{n^{a}}$.\end{lemma}

This is an immediate consequence of the well-known Bernstein inequality
\begin{equation}
\mathbb{P}\left\{ \left|\sum_{l=1}^{n}z_{l}-\mathbb{E}\left[\sum_{l=1}^{n}z_{l}\right]\right|>t\right\} \leq2\exp\left(-\frac{\frac{1}{2}t^{2}}{\sum_{l=1}^{n}\mathbb{E}\left[z_{l}^{2}\right]+\frac{1}{3}Bt}\right).
\end{equation}

\bibliographystyle{IEEEtran} \bibliographystyle{IEEEtran} \bibliographystyle{IEEEtran}
\bibliography{bibfile_rank}

\bibliographystyle{IEEEtran} \bibliographystyle{IEEEtran} \bibliographystyle{IEEEtran}
\bibliographystyle{IEEEtran} 

\section*{Appendix A.}
\label{app:theorem}

\bibliography{bibfile_rank}

\end{document}